%% file: main.tex
\documentclass{article}

\PassOptionsToPackage{numbers, compress}{natbib}

\usepackage[main,final]{neurips_2026}

\usepackage[utf8]{inputenc} 
\usepackage[T1]{fontenc}    
\usepackage{hyperref}       
\usepackage{url}            
\usepackage{booktabs}       
\usepackage{amsfonts}       
\usepackage{nicefrac}       
\usepackage{microtype}      
\usepackage{xcolor}         
\definecolor{linkblue}{RGB}{0,0,128}
\hypersetup{colorlinks=true,allcolors=linkblue}
\usepackage{graphicx}
\usepackage{wrapfig}
\usepackage{capt-of}
\usepackage{placeins}
\usepackage{amsmath}
\usepackage{amsthm}
\usepackage{multirow}
\usepackage{makecell}

\theoremstyle{plain}
\newtheorem{theorem}{Theorem}[section]
\newtheorem{proposition}[theorem]{Proposition}

\theoremstyle{definition}

\theoremstyle{remark}

\title{Compact SO(3) Equivariant Atomistic Foundation Models via Structural Pruning}

\author{
  Chen Wang$^{1,2}$ \quad Siyu Hu$^{1,*}$ \quad Guangming Tan$^{1,2,*}$ \quad Weile Jia$^{1,2,*}$ \\
  $^{1}$State Key Lab of Processors, Institute of Computing Technology, \\
  Chinese Academy of Sciences, Beijing, China \\
  $^{2}$University of Chinese Academy of Sciences, Beijing, China \\
  $^{*}$Corresponding authors: \texttt{\{husiyu,tgm,jiaweile\}@ict.ac.cn}
}

\begin{document}

\maketitle

\begin{abstract}
    SO(3) equivariant graph neural networks have become the dominant paradigm for atomistic foundation models, achieving high accuracy and data efficiency by building rotational symmetry directly into the architecture. Yet the computational cost of their higher-order tensor operations creates a tough trade-off between model accuracy and inference efficiency. In this paper, we propose a structural pruning method for SO(3) equivariant atomistic foundation models to bridge this accuracy-efficiency gap. The pruning is applied along the channel and order dimensions, with each irreducible representation kept or removed as a complete block, thereby retaining SO(3) equivariance. Starting from a large checkpoint, the pruned model substantially reduces the inference cost while retaining higher accuracy than an independently trained small model. The pruned MACE-MP model outperforms the official from-scratch trained small model on 7 of 9 metrics on the Matbench Discovery leaderboard. In terms of efficiency, compressed MACE-MP and MACE-OFF models contain 1.5$\times$ to 4$\times$ fewer parameters and require 2.5$\times$ to 4$\times$ less pre-training compute than training a small model from scratch. For downstream applications, fine-tuning the pruned model reduces energy and force errors by 70.1\% and 34.4\% compared to training task-specific models from scratch across eight representative downstream datasets. We demonstrate that the method generalizes to other SO(3) equivariant architectures (SevenNet, eSCN) and can be combined with quantization and knowledge distillation for further gains.
\end{abstract}

\section{Introduction}

\begin{wrapfigure}[15]{r}{0.50\textwidth}
    \vspace{-1.25\baselineskip}
    \centering
    \includegraphics[width=0.47\textwidth]{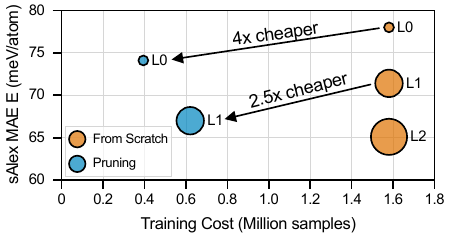}
    \caption{\textbf{Pruning results}. Pruned models outperform scratch-trained models of the same size at 2.5$\times$--4$\times$ lower training cost. Point size indicates parameter count.}
    \label{Fig:Pruning results}
\end{wrapfigure}

Atomistic foundation models have recently emerged as a transformative paradigm in computational chemistry and materials science~\cite{kalita2025machine}. These models~\cite{batatia2023foundation, park_scalable_2024, wood2025family, rhodes2025orb, mazitov2025pet} are available at different scales, placing each at a different point on the accuracy-efficiency frontier. Larger foundation models often achieve stronger benchmark accuracy than smaller ones, but their training cost, inference latency~\cite{leeflashtp}, and memory footprint become major bottlenecks for practical deployment. This persistent tension between the accuracy of large checkpoints and the efficiency of smaller ones remains a central challenge for atomistic foundation models.

\WFclear
\newpage
Among atomistic foundation models, SO(3) equivariant architectures have shown particularly strong accuracy on widely used benchmarks by explicitly encoding rotational symmetry~\cite{thomas2018tensor, geiger2022e3nn, passaro2023reducing}.  This accuracy gain, however, comes with substantial computational overhead: equivariance is commonly implemented through tensor product operations over higher-order geometric features, whose cost scales up to $\mathcal{O}(L_{\max}^6)$, where $L_{\max}$ is the maximum order. This makes the accuracy-efficiency trade-off especially tough for SO(3) equivariant atomistic foundation models.

Model compression is therefore an important route to enhancing model efficiency, with three common paradigms: structural pruning, knowledge distillation (KD)~\cite{gou2021knowledge}, and quantization~\cite{gholami2022survey}. As summarized in Table~\ref{Tab:Compression comparison}, these methods target different aspects of the efficiency problem. Structural pruning removes redundant architectural units from a pre-trained checkpoint, producing a smaller model in the same family with reduced parameters and FLOPs. KD trains a smaller student model to mimic a larger teacher, which can reduce inference cost but typically depends on the chosen student architecture and does not directly reuse the teacher weights. Quantization keeps the architecture unchanged and reduces numerical precision, improving memory and hardware efficiency without reducing the number of parameters or tensor-product paths.

\begin{table}[htbp]
    \caption{\textbf{Comparison of model compression paradigms for atomistic foundation models.} Our structural pruning is distinct from and can be combined with Knowledge Distillation and Quantization.}
    \label{Tab:Compression comparison}
    \centering
    \small
    \setlength{\tabcolsep}{2.5pt}\renewcommand{\arraystretch}{1.15}
    
    \begin{tabular}{lccc}
        \toprule
        \textbf{Criterion} & \textbf{Structural Pruning (Ours)} & \textbf{Knowledge Distillation} & \textbf{Quantization} \\
        \midrule
        Weight source           & Sliced from pre-trained & Random init (student) & Pre-trained (bit-reduced) \\
        Architecture scope      & Same family             & Often cross-arch      & Same arch \\
        Reduces FLOPs \& params & \checkmark              & \checkmark (smaller student) & $\times$ \\
        Reduces bit-precision   & $\times$                & $\times$              & \checkmark \\
        SO(3) equivariance      & Exactly preserved       & Depends on student    & Approx.\ preserved \\
        \bottomrule
    \end{tabular}
    
\end{table}

Among these compression strategies, we focus on structural pruning because it can directly inherit weights from a large checkpoint and produce a smaller pre-trained model, offering a natural way to bridge this accuracy-efficiency gap. However, pruning strategies developed for large language models~\cite{gupta2022compression, xia2023sheared, ashkboos2024slicegpt, men2025shortgpt} cannot be directly transferred to SO(3) equivariant models: unlike scalar-featured networks, these models operate on higher-order geometric tensors, so naively removing feature dimensions breaks equivariance. To the best of our knowledge, existing pruning methods for atomistic foundation models remain coarse-grained, typically removing interaction blocks whose marginal accuracy contribution is limited~\cite{ghunaim2025towards, kong2025scalable}. While such block-level pruning preserves SO(3) equivariance by removing entire equivariant modules, it effectively reduces the receptive field of the model and can be especially harmful for shallow model architectures~\cite{batatia2022mace}.
This raises a central challenge: \textbf{how can we identify removable structures and compress the weights of a large pre-trained SO(3) equivariant model while preserving equivariance and retaining accuracy?}

To address this challenge, we propose a structural pruning framework for SO(3) equivariant atomistic foundation models. Our key contributions are as follows:
\begin{itemize}
    \item We introduce a structural pruning method for higher-order SO(3) equivariant tensors across both channel and order dimensions while preserving SO(3) equivariance.
    \item We propose a four-stage framework to compress foundation models for downstream fine-tuning: (1) calibration, (2) pruning, (3) retraining, and (4) fine-tuning.
    \item We validate our approach on MACE-MP and MACE-OFF, achieving 1.5$\times$ to 4$\times$ parameter reduction and 2.5$\times$ to 4$\times$ training cost reduction (Figure~\ref{Fig:Pruning results}). The pruned models outperform official models trained from scratch on accuracy and achieve 2.7$\times$ inference speedup. When fine-tuned on downstream tasks, they reduce energy and force errors by 70.1\% and 34.4\% compared to training from scratch.
    \item We demonstrate the generalizability of our method across architectures (SevenNet, eSCN) and show that structural pruning can be combined with quantization and knowledge distillation for further efficiency and accuracy gains.
\end{itemize}

Code is available at \url{https://github.com/HPC-AI-Team/SO3-Pruning}.

\section{Related work}

\paragraph{Atomistic foundation models.} Atomistic foundation models are MLIPs trained on massive chemical datasets to predict interatomic potentials across the periodic table. Unlike earlier MLIPs~\cite{wang2018deepmd, schutt2021equivariant, gasteiger2021gemnet, batzner20223} tailored for specific chemical systems, these models~\cite{batatia2023foundation, park_scalable_2024, wood2025family, rhodes2025orb, mazitov2025pet, tan2025high, lysogorskiy2025graph, zhang2506graph, yang2024mattersim, qu2026recipe, zhou2026matris, zhou2025fastchgnet, zhou2026breaking, frank2022so3krates, kabylda2025molecular} are pre-trained on vast, diverse quantum mechanical datasets~\cite{deng2023chgnet, eastman2023spice, barroso2024open, levine2025open, gharakhanyan2026open, sriram2025open} to learn transferable representations of interatomic interactions. By effectively bridging the accuracy of ab initio methods with the efficiency of classical force fields~\cite{schutt2020machine}, these foundation models enable accurate simulations at unprecedented scales~\cite{merchant2023scaling}. In practical applications such as drug discovery and materials design, foundation models can be fine-tuned on domain-specific datasets~\cite{radova2025fine, kaur2025data, 3780338.3782837}. However, their massive parameter counts and computational costs create a deployment bottleneck, motivating our work to develop efficient compression strategies that make these foundation models practical for downstream applications.

\paragraph{Model compression.}
Model compression~\cite{deng2020model} encompasses a suite of techniques, including pruning~\cite{frankle2018lottery}, knowledge distillation (KD)~\cite{gou2021knowledge, ekstrom2023accelerating, amin2025towards, wang2026architecture} and quantization~\cite{gholami2022survey, benoit2025speeding}, designed to reduce the memory footprint and inference latency of deep networks. While distillation requires training student networks from scratch and quantization reduces bit-precision, pruning directly obtains the compressed model weights from the pre-trained model. In this paper, we focus specifically on pruning, as it offers a direct pathway to reduce computational overhead while maintaining the high numerical precision required for MLIPs.

\paragraph{Structural pruning.}
Structural pruning involves removing entire structural units rather than individual weights, ensuring acceleration without specialized sparse kernels. This approach has been widely studied in LLMs~\cite{michel2019sixteen, wang2020structured, ma2023llm, xia2023sheared}. Notably, existing work in MLIPs has largely focused on layer-level pruning strategies~\cite{ghunaim2025towards, kong2025scalable}, which reduce computational cost by decreasing model depth. In contrast, our work extends structural pruning to the domain of SO(3) equivariant GNNs, addressing the unique challenge of pruning coupled irreducible representations connected by tensor products, which standard channel-pruning methods fail to handle.

\section{Method}\label{Sec:Method}

MLIPs learn a parameterized energy $E = \Phi_\theta(\mathbf{R}, \mathbf{Z})$ over atomic positions $\mathbf{R}$ and species $\mathbf{Z}$, with forces derived as $\mathbf{F}_i = -\nabla_{\mathbf{r}_i}\Phi_\theta$. State-of-the-art MLIPs encode physical symmetries via SO(3) equivariant message passing~\cite{thomas2018tensor, geiger2022e3nn, batzner20223}, where node features are higher-order tensors $h_{j,klm}^{(t)}$ indexed by channel $k$ and irreducible representation $(l, m)$, coupled via Clebsch-Gordan products with $\mathcal{O}(L_{\max}^6)$ complexity. Structural pruning~\cite{han2015deep, ma2023llm} removes coherent parameter groups to yield dense, hardware-friendly compressed models. However, directly applying it to SO(3) equivariant architectures is non-trivial, as naive removal of feature dimensions breaks the symmetry of the tensor product path (see Appendix~\ref{AppSec:Background} and \ref{AppSec:Motivation and Theory} for details).

Designing a structural pruning framework for SO(3) equivariant atomistic foundation models requires addressing two key issues: (1) identifying a pruning granularity that strictly preserves the geometric symmetry of the feature space, and (2) establishing a physics-informed importance criterion that captures the contribution of features to both potential energy and forces.

\begin{figure}[ht]
    \centering
    \includegraphics[width=\textwidth]{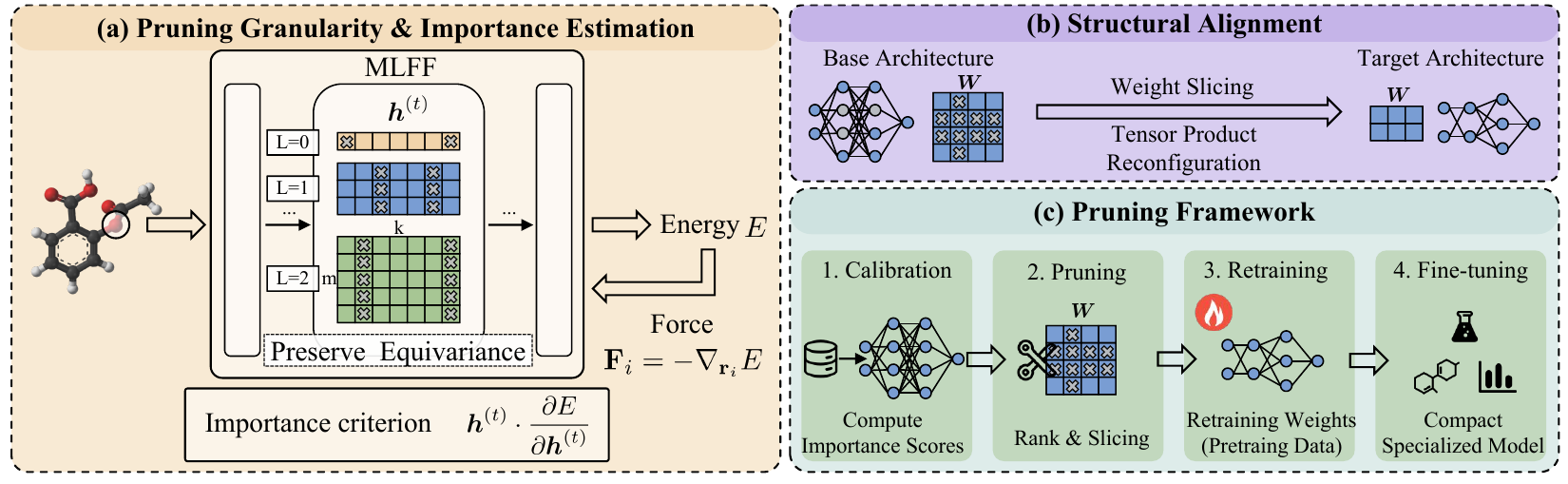}
    \caption{\textbf{Overview of the proposed structural pruning framework.} (a) \textbf{Pruning granularity \& importance estimation:} To preserve SO(3) equivariance, features are pruned as $(k, l)$ blocks based on an energy-force sensitive importance criterion. (b) \textbf{Structural alignment:} The model structure is aligned to the target structure via weight slicing and tensor product reconfiguration. (c) \textbf{Pruning framework:} The four-stage pipeline (Calibration, Pruning, Retraining, Fine-tuning) effectively compresses the foundation model to a compact specialized model.}
    \label{Fig:Pruning framework}
\end{figure}

\subsection{Equivariant pruning granularity}

Standard structural pruning treats channels as independent units. However, in SO(3) equivariant models, the feature space is decomposed into a direct sum of irreducible representations of order $l \in [0, L_{\text{max}}]$. A feature vector at layer $t$, atom $j$, channel $k$, and order $l$ consists of components $h_{j, k l m}^{(t)}$ where $m \in [-l, l]$.

The rotation operation in SO(3) inherently couples all components within the $m$ dimension. Treating these components independently would violate the transformation rules of the irreducible representations, thereby breaking the model's equivariance (see Proposition \ref{prop:equivariance_preservation} in Appendix for a formal proof). Therefore, any structural pruning must treat the subspace spanned by the $2l+1$ components of a specific order $l$ as an atomic unit to avoid breaking SO(3) equivariance. We define our pruning granularity based on the tuple $(k, l)$, representing the $k$-th channel of order $l$.

We introduce a binary mask $\mathbf{z}^{(t)} \in \{0, 1\}^{K \times (L_{\text{max}}+1)}$, where $z_{k, l}^{(t)}$ indicates whether the feature block corresponding to channel $k$ and order $l$ is retained. The pruned feature is formally defined as:
\begin{equation}
    \hat{h}_{j, k l m}^{(t)} = z_{k, l}^{(t)} \cdot h_{j, k l m}^{(t)}.
\end{equation}
This formulation allows for heterogeneous pruning across orders. The model can retain scalar information ($l=0$) for a specific channel while discarding its higher-order geometric tensors ($l>0$), or vice versa. This flexibility is essential for compressing MLIPs, as higher-order interactions are computationally expensive but often sparse in their contribution to the prediction. As shown in Proposition \ref{prop:frequency_l} (Appendix \ref{AppSec:Spectral Interpretation}), higher order $l$ corresponds to higher angular frequencies. While these higher-order interactions capture fine geometric details, they are computationally expensive and often sparse in their contribution to the potential energy of stable systems.

\subsection{Importance estimation}

A robust pruning criterion for MLIPs must account for their dual objective: predicting the potential energy $E$ and the forces $\mathbf{F}_i = -\nabla_{\mathbf{r}_i} E$. According to this, we propose an SO(3) invariant importance criterion based on the sensitivity of the potential energy surface. We approximate the impact of removing a feature block $(k, l)$ on the total energy $E$ using a first-order Taylor expansion. If a feature is pruned (i.e., set to zero), the resulting change in energy $\Delta E$ can be approximated as:
\begin{equation}
    \Delta E \approx \sum_{j, m} \frac{\partial E}{\partial h_{j, k l m}^{(t)}} \cdot (0 - h_{j, k l m}^{(t)}) = - \sum_{j, m} \frac{\partial E}{\partial h_{j, k l m}^{(t)}} \cdot h_{j, k l m}^{(t)}.
\end{equation}

Here, the term $\frac{\partial E}{\partial h_{j, k l m}^{(t)}}$ represents the gradient of the energy with respect to the feature. This product effectively combines the feature's magnitude with the model's sensitivity to that feature.

This gradient-based approach is uniquely advantageous for MLIPs due to the intrinsic relationship between energy and forces. Since calculating forces requires backpropagating the gradient of $E$ to the atomic positions $\mathbf{R}$, the intermediate gradients $\frac{\partial E}{\partial h}$ are computed during the standard force inference. Consequently, our criterion incurs negligible computational overhead. Furthermore, because forces are derived from the energy, features that significantly influence $E$ (high gradient magnitude) are potentially critical for accurate force prediction.

We aggregate this importance score over the representation dimension $m$ and all atoms $j$, averaged over a calibration dataset $\mathcal{D}$. The final importance score $I_{k, l}^{(t)}$ is defined as:
\begin{equation}
    I_{k, l}^{(t)} = \mathbb{E}_{\mathbf{R}, \mathbf{Z} \sim \mathcal{D}} \left| \sum_{j, m} \frac{\partial E}{\partial h_{j, k l m}^{(t)}} \cdot h_{j, k l m}^{(t)} \right|.
\end{equation}
This criterion is SO(3) invariant (see Appendix~\ref{AppSec:Rotational Invariance} for proof), ensuring that pruning decisions are consistent under rotations of the atomic configuration.

By retaining features with the highest $I_{k, l}^{(t)}$, we ensure the preservation of the structural pathways important for constructing both the potential energy surface and the associated force field. The effectiveness of this criterion over activation-only and magnitude-only alternatives is validated in Appendix~\ref{AppSec:Ablation Importance}.

\subsection{Structural alignment} \label{Sec:Structural Alignment}

Given a pre-defined target architecture (i.e., the target number of channels for each order $l$ at each layer), we perform pruning by selecting the top-ranked feature blocks within each layer. However, simply masking features is insufficient for acceleration. The model structure must be physically altered to match the reduced feature dimensions. We enforce structural consistency through the following alignment steps:

\begin{enumerate}
    \item \textbf{Tensor Product Reconfiguration:} The tensor product aggregates features via coupled paths. When the feature dimension of an input order $l_2$ is reduced, the corresponding radial weights $R_{k l_1 l_2 l_3}^{(t)}$ and Clebsch-Gordan paths must be indexed and sliced to align with the retained active channels. This ensures that the computational graph strictly corresponds to the dense, compacted tensors of the target architecture.
    \item \textbf{Weight Slicing:} For the linear transformation within the equivariant blocks, the weights are parameterized as $W_{k \tilde{k} l}^{(t)}$, mapping the $\tilde{k}$-th input channel of order $l$ to the $k$-th output channel. To align with the pruned architecture, we physically slice these weights. Specifically, we retain only the rows corresponding to the active output channels and the columns corresponding to the active input channels. The remaining entries form a smaller, dense weight tensor, effectively reducing the computational cost of the linear projection.
\end{enumerate}

This process ensures that the pruned model is not a sparse version of the original but a dense model with a smaller architecture, ready for efficient inference on standard hardware.

\subsection{Pruning framework}

Our framework follows the four-stage pipeline illustrated in Figure~\ref{Fig:Pruning framework}:
\begin{enumerate}
    \item \textbf{Calibration:} We pass a small subset of pre-training data through the pre-trained foundation model. We compute both the forward pass (Energy) and the backward pass (Forces) to accumulate the importance scores $I_{k, l}^{(t)}$.
    \item \textbf{Target-Driven Pruning:} We define a target compact architecture, specifying the desired number of channels for each order $l$ at each layer. For every layer, we rank the feature blocks by their importance scores $I_{k, l}^{(t)}$ and retain the top blocks to match the target specification. The model parameters are then physically sliced (as described in Sec. \ref{Sec:Structural Alignment}) to instantiate the compact model.
    \item \textbf{Lightweight Retraining:} To recover from the approximation error introduced by pruning, we perform a one-time retraining phase on a subset of the pre-training data. This allows the remaining weights to adapt to the reduced architecture before task-specific fine-tuning.
    \item \textbf{Fine-tuning:} Finally, we fine-tune the compact foundation model on a specific dataset to obtain a compact specialized model.
\end{enumerate}

\section{Experiments}

We evaluate the proposed SO(3) equivariant structural pruning framework primarily on MACE~\cite{batatia2022mace}, and additionally demonstrate its architectural generalizability on SevenNet~\cite{park_scalable_2024} and eSCN~\cite{passaro2023reducing}. Our experiments address four key questions: (1) Do pruned models achieve higher accuracy than same-size models trained from scratch? (2) Do they deliver practical efficiency gains in training cost and inference speed? (3) Do they serve as effective initializations for downstream fine-tuning? (4) Can structural pruning be combined with other compression methods, such as quantization and knowledge distillation?

\subsection{Experimental setup}

\paragraph{Foundation models.} Our experiments utilize the MACE~\cite{batatia2022mace} architecture. We examine two foundation model families: MACE-MP-0~\cite{batatia2023foundation} for inorganic materials (pre-trained on MPtrj~\cite{deng2023chgnet}, $\sim$1.5M structures, 89 elements) and MACE-OFF~\cite{kovacs2025mace} for organic chemistry (pre-trained on SPICE~\cite{eastman2023spice}, $\sim$1.0M conformations). Both families offer official checkpoints of varying sizes, providing baselines to benchmark our pruning method.

\paragraph{Downstream tasks.} For inorganic systems, we evaluate on eight diverse datasets (SSE, H2O, AgAu, AlMgCu, Cu, Ti, V, W) covering materials under varied temperature and pressure conditions. For organic systems, we evaluate on 3BPA~\cite{kovacs2021linear}, AcAc~\cite{batatia2022design}, and rMD17~\cite{christensen2020role}. Dataset details are provided in Appendix~\ref{AppSec:Datasets}.

\paragraph{Implementation details.} Further details regarding the calibration dataset construction, specific architectural handling during pruning, retraining hyperparameters, and ablation studies are provided in Appendix~\ref{AppSec:Experimental Details}. During the retraining phase, we maintain the same number of training epochs as the baseline models trained from scratch. Consequently, a reduction in training data size translates linearly to a reduction in total computational cost (GPU hours).

\subsection{Accuracy}

We evaluate whether pruned models produce higher-quality foundation models than those trained from scratch at the same parameter budget. We prune larger MACE-MP and MACE-OFF models into smaller target architectures using the official MACE settings, enabling a direct comparison with from-scratch baselines of identical size under the same evaluation protocol.

\begin{table}[ht]
    \caption{\textbf{Accuracy of pruning methods on MACE-MP.} The table compares the performance of pruned models against baselines trained from scratch. Metrics include MAE for Energy (E, meV/atom), Force (F, meV/\AA), and Stress (S, meV/$\text{\AA}^3$) on MPtrj, and Energy (E, meV/atom) on the sAlex subset. Data indicates the fraction of training data used. Values in parentheses show the relative change in MAE compared to from-scratch trained models of the same size. Bold values indicate that the pruned model outperforms the from-scratch trained baseline.}
    \label{Tab:Data efficiency of pruning methods on MACE-MP}
    \centering
    \small
    \setlength{\tabcolsep}{2.5pt}\renewcommand{\arraystretch}{1.15}

    \begin{tabular*}{\textwidth}{@{\extracolsep{\fill}}lcccccccc@{}}
        \toprule
        \multirow{2}{*}{\textbf{Method}} & \multirow{2}{*}{\makecell{\textbf{Base}\\\textbf{Model}}} & \multirow{2}{*}{\makecell{\textbf{Pruned}\\\textbf{Model}}} & \multirow{2}{*}{\textbf{Params}} & \multirow{2}{*}{\textbf{Data}} & \multicolumn{3}{c}{\textbf{MPtrj}} & \textbf{sAlex} \\
        \cmidrule(lr){6-8}\cmidrule(l){9-9}
        & & & & & \textbf{MAE E} & \textbf{MAE F} & \textbf{MAE S} & \textbf{MAE E} \\
        \midrule
        \multirow{4}{*}{\makecell[l]{From\\Scratch}} & Layer2 L1 C256 & \multicolumn{1}{c}{\multirow{4}{*}{/}} & 15,847,440 & \multirow{4}{*}{100\%} & / & / & / & 81.7 \\
        & Layer2 L2 C128 & & 5,725,072 & & 22.6 & 42.2 & 1.6 & 65.1 \\
        & Layer2 L1 C128 & & 4,688,656 & & 25.8 & 46.7 & 1.7 & 71.4 \\
        & Layer2 L0 C128 & & 3,847,696 & & 29.6 & 52.3 & 1.8 & 78.0 \\
        \midrule
        \multirow{3.5}{*}{L Pruning} & \multirow{3.5}{*}{Layer2 L2 C128} & \multirow{3.5}{*}{Layer2 L0 C128} & \multirow{3.5}{*}{3,847,696} & \textbf{25\%} & \makecell{30.9\\(+4.4\%)} & \makecell{\textbf{50.8}\\\textbf{(-2.9\%)}} & \makecell{1.9\\(+5.5\%)} & \makecell{\textbf{74.1}\\\textbf{(-5.0\%)}} \\
        & & & & 100\% & \makecell{\textbf{25.5}\\\textbf{(-13.9\%)}} & \makecell{\textbf{47.3}\\\textbf{(-9.6\%)}} & \makecell{\textbf{1.7}\\\textbf{(-5.6\%)}} & \makecell{\textbf{71.9}\\\textbf{(-7.8\%)}} \\
        \midrule
        \makecell[l]{Channel\\Pruning} & Layer2 L1 C256 & Layer2 L1 C128 & 4,688,656 & \textbf{40\%} & \makecell{\textbf{16.4}\\\textbf{(-36.4\%)}} & \makecell{\textbf{45.8}\\\textbf{(-1.9\%)}} & \makecell{1.8\\(+5.9\%)} & \makecell{\textbf{67.0}\\\textbf{(-6.2\%)}} \\
        \midrule
        \makecell[l]{L \& Channel\\Pruning} & Layer2 L1 C256 & Layer2 L0 C128 & 3,847,696 & \textbf{40\%} & \makecell{\textbf{18.9}\\\textbf{(-36.1\%)}} & \makecell{\textbf{51.7}\\\textbf{(-1.1\%)}} & \makecell{1.9\\(+5.5\%)} & \makecell{\textbf{67.2}\\\textbf{(-13.8\%)}} \\
        \bottomrule
    \end{tabular*}
    \begin{flushleft}{\footnotesize \textit{Notation:} Layer$x$ L$y$ C$z$: $x$ message-passing layers, max order $y$, $z$ channels.}\end{flushleft}
\end{table}

\paragraph{MACE-MP results:}
Table~\ref{Tab:Data efficiency of pruning methods on MACE-MP} compares pruned models against from-scratch baselines of the same target architecture. Our method compresses model parameters up to 4$\times$. Pruned models achieve lower Force MAE and substantially better out-of-distribution performance on the sAlex~\cite{schmidt2024improving, barroso2024open, fulearning} dataset compared to from-scratch baselines of the same size, indicating that the pruned model effectively inherits the generalized feature representations of the larger parent model.

\paragraph{MACE-OFF results:} Table~\ref{Tab:Accuracy of pruning methods on MACE-OFF} presents similar findings for organic molecules. Across diverse test subsets, the pruned model retrained with 100\% of the data outperforms the from-scratch baseline of the same size on all metrics, and already matches it with only 50\% of the data, demonstrating that the pruning process yields a higher-quality compact model while preserving broad coverage of chemical space.

\begin{table}[ht]
    \caption{\textbf{Accuracy of pruning methods on MACE-OFF.} The table compares the performance of pruned models against baselines trained from scratch. Metrics include MAE for Energy (E, meV/atom) and Forces (F, meV/\AA) across diverse molecular subsets. Data indicates the fraction of training data used. Bold values indicate that the pruned model outperforms the from-scratch trained baseline of the same size.}
    \label{Tab:Accuracy of pruning methods on MACE-OFF}
    \centering
    \small
    \setlength{\tabcolsep}{2.5pt}\renewcommand{\arraystretch}{1.15}

    \begin{tabular*}{\textwidth}{@{\extracolsep{\fill}}lcccccccc@{}}
    \toprule
    \textbf{Method} & \multicolumn{4}{c}{From Scratch} & \multicolumn{4}{c}{L \& Channel Pruning} \\
    \cmidrule(lr){2-5}\cmidrule(l){6-9}
    \textbf{Params} & \multicolumn{2}{c}{4,707,312} & \multicolumn{2}{c}{1,428,368} & \multicolumn{2}{c}{1,428,368} & \multicolumn{2}{c}{1,428,368} \\
    \textbf{Data} & \multicolumn{2}{c}{100\%} & \multicolumn{2}{c}{100\%} & \multicolumn{2}{c}{\textbf{50\%}} & \multicolumn{2}{c}{100\%} \\
    \cmidrule(lr){2-3}\cmidrule(lr){4-5}\cmidrule(lr){6-7}\cmidrule(l){8-9}
    \textbf{Subset} & \textbf{E} & \textbf{F} & \textbf{E} & \textbf{F} & \textbf{E} & \textbf{F} & \textbf{E} & \textbf{F} \\
    \midrule
    DES370K Dimers & 0.5 & 6.6 & 0.6 & 9.2 & 0.6 & 9.8 & \textbf{0.5} & \textbf{9.0} \\
    DES370K Monomers & 0.6 & 6.6 & 0.6 & 9.5 & 0.6 & 10.1 & \textbf{0.5} & \textbf{9.3} \\
    Dipeptides & 0.4 & 10.2 & 0.5 & 14.5 & \textbf{0.5} & 15.3 & \textbf{0.5} & \textbf{14.3} \\
    PubChem & 0.9 & 14.8 & 0.9 & 22.0 & 0.9 & 23.2 & \textbf{0.8} & \textbf{21.7} \\
    QMugs & 0.5 & 17.0 & 0.6 & 24.2 & \textbf{0.6} & 25.0 & \textbf{0.6} & \textbf{23.5} \\
    Solvated Amino Acids & 1.0 & 19.4 & 1.3 & 23.8 & \textbf{1.2} & 24.3 & \textbf{1.2} & \textbf{22.6} \\
    Water & 0.8 & 13.6 & 0.8 & 16.1 & 0.8 & 16.4 & \textbf{0.7} & \textbf{15.5} \\
    \bottomrule
    \end{tabular*}
\end{table}

\paragraph{Generalization to large-scale benchmarks:} To further validate that the accuracy gains extend beyond the pre-training distribution, we evaluate on Matbench Discovery~\cite{riebesell2023matbench}, a large-scale screening benchmark across 250,000+ candidate materials. Our pruned MACE-MP model (Layer2 L0 C128) surpasses the official MACE-MP-0 checkpoint (Layer2 L1 C128) on 7/9 metrics using only 25\% of MPtrj (Table~\ref{Tab:Matbench Discovery}), demonstrating that the pruned model is not only comparable to but can exceed the accuracy of from-scratch-trained foundation models in thermodynamic stability prediction.

\begin{table}[ht]
    \caption{\textbf{Matbench Discovery results.} Our pruned model is compared against the official MACE-MP-0 baseline. Bold values indicate the better result on each metric. $\uparrow$: higher is better; $\downarrow$: lower is better.}
    \label{Tab:Matbench Discovery}
    \centering
    \small
    \setlength{\tabcolsep}{2.5pt}\renewcommand{\arraystretch}{1.15}
    
    \begin{tabular}{lccccccccc}
        \toprule
        \textbf{Model} & \textbf{F1 $\uparrow$} & \textbf{DAF $\uparrow$} & \textbf{Prec. $\uparrow$} & \textbf{Recall $\uparrow$} & \textbf{Acc. $\uparrow$} & \textbf{MAE $\downarrow$} & \textbf{R$^2$ $\uparrow$} & \textbf{$\kappa_{\mathrm{SRME}}$ $\downarrow$} & \textbf{RMSD $\downarrow$} \\
        \midrule
        MACE-MP-0 (Official) & 0.669 & 3.777 & 0.577 & \textbf{0.796} & 0.878 & 0.057 & 0.697 & 0.682 & \textbf{0.0915} \\
        Pruned (Ours)        & \textbf{0.680} & \textbf{3.889} & \textbf{0.595} & 0.794 & \textbf{0.884} & \textbf{0.056} & \textbf{0.717} & \textbf{0.625} & 0.0917 \\
        \bottomrule
    \end{tabular}
    
\end{table}

\paragraph{Architectural generalizability:} To verify that our framework applies broadly to SO(3) equivariant architectures beyond MACE, we evaluate on SevenNet~\cite{park_scalable_2024} and eSCN~\cite{passaro2023reducing}. As summarized in Table~\ref{Tab:Architecture generalization summary}, the pruned SevenNet surpasses the scratch baseline Force MAE with only 25\% of MPtrj, and the pruned eSCN trained on 1M OC20~\cite{chanussot2021open} samples surpasses the scratch baseline trained on 2M, achieving a $2\times$ reduction in data requirements. Full details are in Appendix~\ref{AppSec:Applicability of the Proposed Method}.

\begin{table}[ht]
    \caption{\textbf{Generalizability across architectures.} The pruned model matches the from-scratch baseline with a fraction of the training data.}
    \label{Tab:Architecture generalization summary}
    \centering
    \small
    \setlength{\tabcolsep}{2.5pt}\renewcommand{\arraystretch}{1.15}
    
    \begin{tabular}{llcccc}
        \toprule
        \multirow{2}{*}{\textbf{Architecture}} & \multirow{2}{*}{\textbf{Dataset}} & \multicolumn{2}{c}{\textbf{From Scratch}} & \multicolumn{2}{c}{\textbf{Pruned}} \\
        \cmidrule(lr){3-4} \cmidrule(lr){5-6}
        & & \textbf{Data} & \textbf{Force MAE (meV/\AA)} & \textbf{Data} & \textbf{Force MAE (meV/\AA)} \\
        \midrule
        SevenNet & MPtrj      & 100\% & 40   & 25\% & \textbf{39}   \\
        eSCN     & OC20-2M    & 100\% & 19.1 & 50\% & \textbf{18.4} \\
        \bottomrule
    \end{tabular}
    
\end{table}

\FloatBarrier
\subsection{Efficiency}

\noindent\begin{minipage}[t]{0.50\textwidth}
\vspace{0pt}
Compact architectures are only practically useful if they deliver concrete reductions in computational cost. We evaluate the efficiency of the pruned models along two axes: end-to-end training cost and inference throughput.

\paragraph{Training cost:} As shown in Figure~\ref{Fig:Comparison of energy and force MAE trends between From Scratch and Pruning cross different training data sizes}, the pruned model consistently outperforms from-scratch training across all data scales, and matches the scratch baseline accuracy with only 25--40\% of the data. This translates directly to lower training cost. The full pipeline (calibration, pruning, and retraining on 25\% of MPtrj) requires approximately 181 GPU-hours on an A800, compared to $\sim$750 GPU-hours for pre-training a model from scratch, confirming a \textbf{4$\times$} reduction in training cost (Appendix~\ref{AppSec:Pipeline Cost}).

\paragraph{Inference efficiency:} Using the MPtrj dataset as an example, on an NVIDIA A800 GPU, our pruned models achieve a 2.7$\times$ throughput improvement and a 5.7$\times$ reduction in peak memory usage compared to the original foundation models. On an RTX 3090, structural pruning yields a 3.0$\times$ throughput improvement, confirming that the efficiency gains are hardware-agnostic (Appendix~\ref{AppSec:Throughput and Memory Usage}).

\end{minipage}\hfill
\begin{minipage}[t]{0.47\textwidth}
    \vspace{-9pt}
    \centering
    \includegraphics[width=\linewidth]{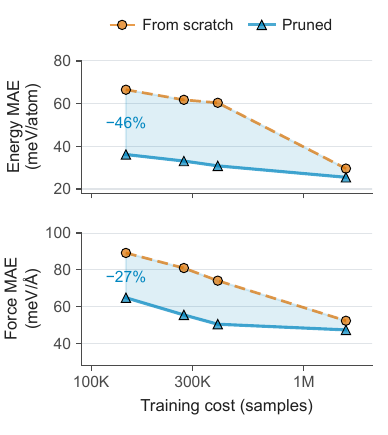}
    \par\vspace{7pt}
    \captionof{figure}{\textbf{Energy and force MAE vs.\ training data size on MACE-MP} (Layer2 L2 C128 $\to$ Layer2 L0 C128). Shading indicates the MAE gap between the two models. The pruned model outperforms training from scratch across all data scales.}
    \label{Fig:Comparison of energy and force MAE trends between From Scratch and Pruning cross different training data sizes}
\end{minipage}
\par

\subsection{Application}

We evaluate the utility of the pruned models as initializations for downstream fine-tuning. A key question is whether a pruned-then-retrained model, despite using only a fraction of the pre-training data, can serve as a good starting point.

\begin{figure}[ht]
    \centering
    \includegraphics[width=\textwidth]{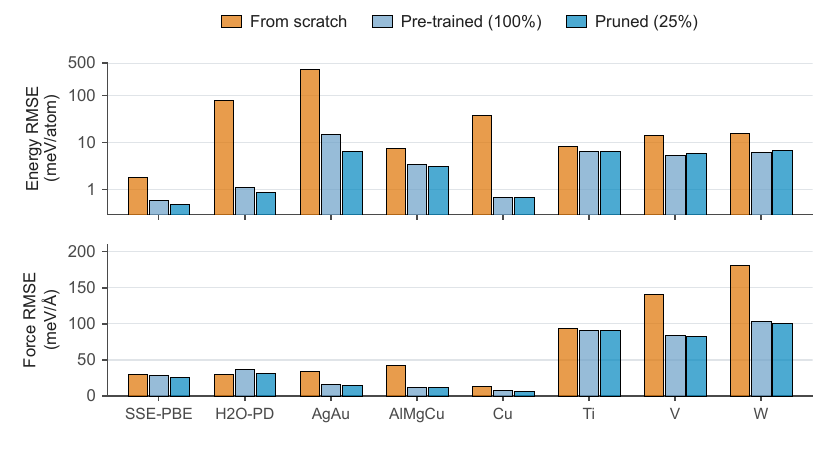}
    \caption{\textbf{Downstream fine-tuning performance on eight inorganic datasets.} Bars show energy and force RMSE for three initialization strategies. The energy axis uses a logarithmic scale, while the force axis is linear. Lower values indicate better accuracy. Averaged equally across datasets, pruning reduces energy and force RMSE by 70.1\% and 34.4\% relative to training from scratch. Percentages in the legend denote the fraction of MPtrj used before fine-tuning. Full RMSE values are reported in Table~\ref{Tab:Effectiveness of pruned MACE-MP model as initialization for downstream fine-tuning} in Appendix~\ref{AppSec:MACE-MP Downstream}.}
    \label{Fig:Downstream Fine-tuning}
\end{figure}

\paragraph{Inorganic materials:} Figure~\ref{Fig:Downstream Fine-tuning} compares energy and force RMSE after fine-tuning on eight inorganic datasets. We compare three initialization strategies: Random Initialization (From Scratch), an official pre-trained checkpoint (using 100\% MPtrj data), and our pruned checkpoint (using only 25\% MPtrj data). The results show that the Pruned (25\%) initialization significantly outperforms a small model trained from scratch, reducing energy and force errors by 70.1\% and 34.4\% respectively, while remaining highly competitive with the full pre-training baseline that uses 4$\times$ more data. To further validate the deployment readiness of the fine-tuned model, we run MD simulations of liquid water. The fine-tuned pruned model achieves energy drift ${<}0.01$ meV/atom/ps and faithfully reproduces structural properties (Appendix~\ref{AppSec:MD Simulation}).

\FloatBarrier
\paragraph{Organic molecules:} For molecular tasks on 3BPA, AcAc, and rMD17 (see Appendix~\ref{AppSec:MACE-OFF Downstream}), pruned MACE-OFF models match or exceed from-scratch baselines across temperature regimes and dihedral configurations, confirming that compact specialized models retain the expressivity needed for complex potential energy surfaces.

Overall, these experiments confirm that our structural pruning framework successfully reduces pre-training cost while retaining transferability. By preserving the critical weights of the large model, the compressed models serve as highly effective, lightweight starting points for downstream applications.

\subsection{Combination with other compression methods}

As noted in Table~\ref{Tab:Compression comparison}, structural pruning can be combined with both quantization and knowledge distillation. We verify empirically that these techniques can be combined with our method to achieve compounding gains beyond what each method achieves alone.

\paragraph{Combination with quantization:} As structural pruning and quantization reduce different aspects of model cost (FLOPs/params vs.\ bit-precision), they can be combined for compounding gains. Table~\ref{Tab:Combined Pruning and Quantization Speedup} reports peak throughput on an RTX 3090 for the base model and pruned models under FP64 and FP32 precision. Structural pruning alone yields a 3.0$\times$ speedup; reducing precision to FP32 alone yields 2.0$\times$; applying both together achieves 9.6$\times$ combined throughput improvement (2,728 to 26,144 atom/s).

\begin{table}[ht]
    \caption{\textbf{Combined speedup from structural pruning and quantization on RTX 3090.} Speedup is relative to the base FP64 baseline.}
    \label{Tab:Combined Pruning and Quantization Speedup}
    \centering
    \small
    \setlength{\tabcolsep}{2.5pt}\renewcommand{\arraystretch}{1.15}
    \begin{tabular}{lccc}
        \toprule
        \textbf{Model} & \textbf{FP64 (atom/s)} & \textbf{FP32 (atom/s)} & \makecell{\textbf{Speedup (vs.\ Base FP64)}\\\textbf{FP64 / FP32}} \\
        \midrule
        Base (Layer2 L2 C128)   & 2,728  & 5,576  & 1.0$\times$ / 2.0$\times$ \\
        Pruned (Layer2 L1 C128) & 4,976  & 12,268 & 1.8$\times$ / 4.5$\times$ \\
        Pruned (Layer2 L0 C128) & 8,271  & 26,144 & 3.0$\times$ / \textbf{9.6}$\times$ \\
        \bottomrule
    \end{tabular}
\end{table}

\paragraph{Combination with knowledge distillation:} Structural pruning and knowledge distillation address different aspects of model compression: pruning inherits weights from a large pre-trained model to produce a compact initialization, while KD transfers knowledge from a teacher to a student through soft labels during training. Here we explore the combination of our structural pruning method with Hessian-based KD~\cite{amin2025towards}, which uses second-order information from the teacher model to guide student training. Table~\ref{Tab:Hessian KD Comparison} compares two initialization strategies for training a Layer2 L0 C128 student model with Hessian labels generated from a Layer2 L2 C128 teacher: random initialization (From Scratch) versus pruning-based initialization (Pruned). With 9\% of MPtrj, the pruned initialization without KD (Force MAE 64.8) already outperforms random initialization with Hessian KD (68.4). Combining pruned initialization with Hessian KD further reduces it to 64.5. With 25\% of MPtrj, pruned initialization alone reaches 50.8, and adding Hessian KD yields 50.7. These results demonstrate that structural pruning and KD are complementary: pruning provides a better starting point, while KD further refines the model through teacher guidance.

\begin{table}[ht]
    \caption{\textbf{Effect of Hessian-based KD combined with structural pruning} (target: Layer2 L0 C128). MAE for Energy (E, meV/atom), Force (F, meV/\AA), Stress (S, meV/\AA$^3$) on MPtrj.}
    \label{Tab:Hessian KD Comparison}
    \centering
    \small
    \setlength{\tabcolsep}{2.5pt}\renewcommand{\arraystretch}{1.15}
    \begin{tabular}{llcccc}
        \toprule
        \textbf{Init.} & \textbf{Objective} & \textbf{Data} & \textbf{MAE E} & \textbf{MAE F} & \textbf{MAE S} \\
        \midrule
        From Scratch   & Standard     & 9\%  & 66.5 & 89.2 & 2.7 \\
        From Scratch   & +Hessian KD  & 9\%  & \textbf{37.2} & \textbf{68.4} & \textbf{2.3} \\
        \midrule
        Pruned      & Standard     & 9\%  & 36.2 & 64.8 & \textbf{2.2} \\
        Pruned      & +Hessian KD  & 9\%  & \textbf{32.8} & \textbf{64.5} & \textbf{2.2} \\
        \midrule
        Pruned      & Standard     & 25\% & 30.9  & 50.8  & \textbf{1.9}  \\
        Pruned      & +Hessian KD  & 25\% & \textbf{23.6}  & \textbf{50.7}  & \textbf{1.9}  \\
        \bottomrule
    \end{tabular}
\end{table}

\FloatBarrier
These results confirm that our structural pruning method can be combined with both quantization and knowledge distillation. Because each technique addresses a fundamentally different axis of model cost (parameter count, bit-precision, and generalization gap, respectively), their benefits compound rather than overlap. This complementarity means practitioners can freely compose these techniques according to their deployment constraints, applying any subset of the three without conflict.

\section{Conclusion}

We present a structural pruning framework for SO(3) equivariant atomistic foundation models. By pruning at the $(k, l)$-block granularity with an SO(3) invariant importance criterion, our method produces compact foundation models that strictly preserve SO(3) equivariance and achieve 1.5$\times$--4$\times$ parameter reduction. From the accuracy perspective, pruned models outperform same-size models trained from scratch on both in-distribution and out-of-distribution benchmarks. From the efficiency perspective, they reduce training cost by 2.5$\times$--4$\times$ and achieve 2.7$\times$ inference speedup. From the application perspective, they serve as superior initializations for downstream fine-tuning, reducing energy and force errors by 70.1\% and 34.4\% compared to training from scratch. Structural pruning can be combined with quantization and knowledge distillation for compounding gains.

The core principle of our method, evaluating the importance of coupled irreducible representations, is broadly applicable to any equivariant architecture that represents features as sums of SO(3) irreducible representations, as validated on MACE, SevenNet, and eSCN. A current limitation is that the target pruned architecture must be specified in advance. Incorporating neural architecture search to automatically discover optimal configurations is a promising direction for future work, enabling more flexible compression tailored to specific deployment constraints.

\renewcommand{\acksection}{\section*{Acknowledgments}}
\begin{ack}
This work is supported by the following funding: National Natural Science Foundation of China (92270206, 62372435, T2125013), and the Innovation Funding of ICT, CAS. Part of the training is performed on the robotic AI-Scientist platform of Chinese Academy of Sciences. The authors declare no competing interests.
\end{ack}

{
\small

\bibliographystyle{unsrtnat}
\bibliography{ref}
}


\newpage
\appendix

\section{Background} \label{AppSec:Background}

\paragraph{Machine-learning interatomic potentials (MLIPs).}

The primary objective of MLIPs is to approximate the Potential Energy Surface (PES) of an atomic system with the accuracy of \textit{ab initio} quantum mechanical methods but at a fraction of the computational cost. Consider a system composed of $N$ atoms, defined by their Cartesian coordinates $\mathbf{R} = \{\mathbf{r}_1, \dots, \mathbf{r}_N\} \in \mathbb{R}^{N \times 3}$ and their chemical species (atomic numbers) $\mathbf{Z} = \{z_1, \dots, z_N\} \in \mathbb{Z}^N$. An MLIP learns a parameterized function $E = \Phi_\theta(\mathbf{R}, \mathbf{Z})$ that predicts the total potential energy $E$ of the system.

Additionally, these models also provide accurate predictions of interatomic forces to drive MD simulations. Since the force field is conservative, the force $\mathbf{F}_i$ acting on the $i$-th atom is derived as the negative gradient of the energy with respect to its position:
\begin{equation}
    \mathbf{F}_i = -\nabla_{\mathbf{r}_i} \Phi_\theta(\mathbf{R}, \mathbf{Z}).
\end{equation}

Atomistic foundation models extend this framework by training $\Phi_\theta$ on massive datasets spanning diverse chemical spaces to achieve broad generalizability across unseen systems. To ensure physical consistency and data efficiency, the architecture of $\Phi_\theta$ is typically designed to respect the geometric symmetries of the 3D space.

\paragraph{SO(3) equivariant message passing.}

MLIPs typically ensure physical symmetries by constructing an energy function $E$ that is SE(3) (translational and rotational) invariant. Since forces are derived via gradients, the invariance of $E$ naturally guarantees the equivariance of $\mathbf{F}_i$. While early architectures relied on invariant scalars (e.g., distances, angles), recent research indicates that incorporating SO(3) (rotational) equivariant geometric tensors significantly improves data efficiency and accuracy~\cite{geiger2022e3nn, batzner20223}.

In this framework, node features are generalized to include higher-order tensors. We denote the feature component of atom $j$ at layer $t$ as $h_{j, k l m}^{(t)}$, where $k$ indexes the feature channel, and $l, m$ denote the order and representation index of the SO(3) irreducible representation. The SO(3) Equivariant Message Passing aggregates information from neighbors via a tensor product~\cite{thomas2018tensor}. The message for atom $i$ at channel $k$ with target order $l_3$ and representation $m_3$ is given by:
\begin{equation}
    \sum_{l_1 m_1, l_2 m_2} C_{l_1 m_1, l_2 m_2}^{l_3 m_3} \sum_{j \in \mathcal{N}(i)} R_{k l_1 l_2 l_3}^{(t)}(r_{ji}) Y_{l_1}^{m_1}(\hat{\mathbf{r}}_{ji}) h_{j, k l_2 m_2}^{(t)}.
\end{equation}
Here, $h_{j, k l_2 m_2}^{(t)}$ is the input feature of the neighbor with order $l_2$. The interaction couples these features with the spherical harmonics $Y_{l_1}^{m_1}$ of the edge direction using Clebsch-Gordan coefficients $C_{l_1 m_1, l_2 m_2}^{l_3 m_3}$, which enforce the rules ($|l_1 - l_2| \le l_3 \le (l_1 + l_2)$). The learnable radial function $R_{k l_1 l_2 l_3}^{(t)}$ modulates the interaction strength based on the interatomic distance $r_{ji}$ and the specific channel $k$. This operation has a computational complexity of $\mathcal{O}(L_{\text{max}}^6)$, making higher-order interactions computationally expensive.

The tensor product operates in a channel-wise manner, processing each channel $k$ independently. To enable information exchange across different channels, the architecture incorporates learnable linear layers to mix feature channels while strictly preserving SO(3) equivariance by operating only on features of the same order $l$:
\begin{equation}
    \sum_{\tilde{k}} W_{k \tilde{k} l}^{(t)} h_{j, \tilde{k} l m}^{(t)}.
\end{equation}

\paragraph{Structural pruning.}

Model pruning compresses models by eliminating redundant parameters. While unstructured pruning creates sparse matrices that often require specialized hardware, structural pruning removes coherent groups, such as entire layers or channels. This approach preserves dense matrix structures, enabling direct acceleration on standard hardware. The core of structural pruning lies in identifying and removing the least significant structures based on a specific importance criterion. Common scoring methods include the magnitude~\cite{han2015deep}, activation~\cite{muralidharan2024compact}, or gradient~\cite{ma2023llm}.

Structural pruning generally operates along two dimensions: depth and width. Taking a Multi-Layer Perceptron (MLP) as an example:
\begin{equation}
    \mathbf{W}_n \sigma(\dots \sigma(\mathbf{W}_2 \sigma(\mathbf{W}_1 x))),
\end{equation}
where layer pruning reduces the depth by removing a weight matrix $\mathbf{W}_k$, effectively shortening the inference path. Conversely, channel pruning reduces the dimension of the matrices by removing specific rows in $\mathbf{W}_k$ (and corresponding columns in $\mathbf{W}_{k+1}$), which shrinks the layer size.

However, applying structural pruning to SO(3) equivariant models is non-trivial. As features are higher-order tensors grouped by order $l$ and representation $m$, naive pruning can break the symmetry of the message passing path.

\section{Motivation and theory} \label{AppSec:Motivation and Theory}

\subsection{Depth pruning} \label{AppSec:Depth Pruning}

Message Passing Neural Networks (MPNNs) approximate the potential energy surface of atomic systems by iteratively refining node features through interactions with neighbors. A critical hyperparameter in these architectures is the number of message passing layers, denoted as $T$. This parameter controls not only the depth of the network but also the physical range of atomic interactions the model can capture—its receptive field.

In the MLIPs architecture, the total energy is a sum of atomic energies, which are determined by the local environment of each atom. The extent of this environment is defined by the receptive field $R_{\text{rec}}$. We formalize the relationship between the network depth and the receptive field with the following proposition:

\begin{proposition}
\label{prop:receptive_field}
Let $r_c$ be the radial cutoff distance for a single message passing iteration. For a Message Passing Neural Network with $T$ layers, the total receptive field $R_{\text{rec}}$ of an atom $i$ is given by:
\begin{equation}
R_{\text{rec}} = T \times r_c
\end{equation}
\end{proposition}

\begin{proof}
We proceed by induction on the number of layers $T$.

Base case ($T=1$): In the first layer, atom $i$ receives messages from all neighbors $j$ such that the Euclidean distance $d_{ij} \le r_c$. Thus, the features $\boldsymbol{h}_i^{(1)}$ depend on atoms within a radius $R_{\text{rec}}^{(1)} = r_c$.

Inductive step: Assume that for $T=k$, the features $\boldsymbol{h}_i^{(k)}$ depend on all atoms within a distance $k \times r_c$. In the $(k+1)$-th layer, atom $i$ aggregates messages from neighbors $j$ within distance $r_c$. Each neighbor $j$ has features $\boldsymbol{h}_j^{(k)}$, which depend on atoms within distance $k \times r_c$ from $j$ by the inductive hypothesis.
By the triangle inequality, an atom $m$ influencing $j$ (where $d_{jm} \le k \times r_c$) is at a distance from $i$ satisfying:
\begin{equation}
d_{im} \le d_{ij} + d_{jm} \le r_c + k \times r_c = (k+1)r_c
\end{equation}
Therefore, the features $\boldsymbol{h}_i^{(k+1)}$ depend on atoms within a radius $R_{\text{rec}}^{(k+1)} = (k+1)r_c$.
\end{proof}

This linear scaling implies that reducing the number of layers (Depth Pruning) directly shrinks the physical volume of the atomic environment that the model can perceive. In contrast, reducing the feature channels or order (Width Pruning) simplifies the mathematical description of the interactions within that volume but preserves the spatial scope of the receptive field. Consequently, we hypothesize that Depth Pruning will result in a more significant loss of accuracy compared to Width Pruning, particularly for systems where long-range interactions are non-negligible.

\subsection{Preservation of SO(3) equivariance} \label{AppSec:Equivariance Preservation}

Standard pruning methods often operate on individual elements of the feature maps. We show here that such element-wise pruning breaks SO(3) equivariance, whereas our proposed $(k, l)$-block pruning preserves it.

\begin{proposition} \label{prop:equivariance_preservation}
Let $\mathbf{h}_l \in \mathbb{R}^{2l+1}$ be a feature vector of order $l$. Let $g \in SO(3)$ be a rotation, and $\mathbf{D}^l(g)$ be the Wigner-D matrix acting on $\mathbf{h}_l$.
\begin{enumerate}
    \item \textbf{Element-wise Pruning:} Let $\mathbf{M} = \text{diag}(m_{-l}, \dots, m_l)$ be an arbitrary diagonal binary mask where not all $m_i$ are equal. Applying $\mathbf{M}$ element-wise violates equivariance:
    \begin{equation}
        \mathbf{M} (\mathbf{D}^l(g) \mathbf{h}_l) \neq \mathbf{D}^l(g) (\mathbf{M} \mathbf{h}_l)
    \end{equation}
    \item \textbf{Block-wise Pruning:} Applying a scalar mask $z \in \{0, 1\}$ (our proposed method) preserves equivariance:
    \begin{equation}
        z (\mathbf{D}^l(g) \mathbf{h}_l) = \mathbf{D}^l(g) (z \mathbf{h}_l)
    \end{equation}
\end{enumerate}
\end{proposition}

\begin{proof}
\textbf{Part 1 (Violation):}
The Wigner-D matrix $\mathbf{D}^l(g)$ is a dense $(2l+1) \times (2l+1)$ unitary matrix that mixes all components indexed by $m \in [-l, l]$.
Let $\mathbf{h}' = \mathbf{D}^l(g) \mathbf{h}$. The $i$-th component of the rotated feature is $h'_i = \sum_j D^l(g)_{ij} h_j$.
If we apply the mask $\mathbf{M}$ after rotation, the result is $m_i \sum_j D^l(g)_{ij} h_j$.
If we apply the mask $\mathbf{M}$ before rotation, the rotated masked feature is $\sum_j D^l(g)_{ij} (m_j h_j)$.
For equivariance to hold, we require $m_i \sum_j D^l(g)_{ij} h_j = \sum_j D^l(g)_{ij} m_j h_j$ for all $\mathbf{h}$ and $g$. This implies $m_i = m_j$ for all $i, j$ where $D^l(g)_{ij} \neq 0$. Since $\mathbf{D}^l(g)$ generally has non-zero off-diagonal entries (mixing components), arbitrary element-wise masking ($m_i \neq m_j$) breaks this equality.

\textbf{Part 2 (Preservation):}
Our method treats the subspace of order $l$ as an atomic unit. The mask becomes a scalar $z \in \{0, 1\}$ multiplying the entire vector $\mathbf{h}_l$. Since scalar multiplication commutes with linear transformations:
\begin{equation}
    z (\mathbf{D}^l(g) \mathbf{h}_l) = \mathbf{D}^l(g) (z \mathbf{h}_l)
\end{equation}
Thus, the pruning operation commutes with the group action, satisfying the definition of equivariance.
\end{proof}

Since the irreducible representations of SO(3) inherently mix all $m$ components during rotation, any pruning operation that differentiates between $m$ values within the same order $l$ will inevitably destroy the group structure. Therefore, the proposed $(k, l)$-block serves as the \textit{minimal pruning granularity} capable of strictly preserving SO(3) equivariance.

\subsection{Rotational invariance of the importance criterion} \label{AppSec:Rotational Invariance}

We now prove that our importance criterion $I_{k,l}^{(t)}$ is SO(3) invariant, ensuring that the pruning decisions are consistent under rotations of the atomic configuration.

\begin{proposition} \label{prop:importance_invariance}
The importance score $I_{k,l}^{(t)} = \mathbb{E}_{\mathbf{R}, \mathbf{Z} \sim \mathcal{D}} \left| \sum_{j, m} \frac{\partial E}{\partial h_{j, k l m}^{(t)}} \cdot h_{j, k l m}^{(t)} \right|$ is invariant under SO(3) rotations of the input atomic configuration.
\end{proposition}

\begin{proof}
Let $g \in SO(3)$ be a rotation. Under this rotation, the atomic positions transform as $\mathbf{R} \to g\mathbf{R}$, and the features transform equivariantly: $h_{j,klm}^{(t)} \to \sum_{m'} D^l(g)_{mm'} h_{j,klm'}^{(t)}$, where $\mathbf{D}^l(g)$ is the Wigner-D matrix.

Since the energy $E$ is rotationally invariant by construction, we have $E(g\mathbf{R}, \mathbf{Z}) = E(\mathbf{R}, \mathbf{Z})$. By the chain rule:
\begin{equation}
\frac{\partial E}{\partial h_{j,klm}^{(t)}} \to \sum_{m'} D^l(g)_{mm'} \frac{\partial E}{\partial h_{j,klm'}^{(t)}}
\end{equation}

The inner product in the importance criterion transforms as:
\begin{align}
\sum_m \frac{\partial E}{\partial h_{j,klm}^{(t)}} \cdot h_{j,klm}^{(t)} &\to \sum_m \left(\sum_{m'} D^l(g)_{mm'} \frac{\partial E}{\partial h_{j,klm'}^{(t)}}\right) \left(\sum_{m''} D^l(g)_{mm''} h_{j,klm''}^{(t)}\right) \\
&= \sum_{m,m',m''} D^l(g)_{mm'} D^l(g)_{mm''} \frac{\partial E}{\partial h_{j,klm'}^{(t)}} h_{j,klm''}^{(t)} \\
&= \sum_{m',m''} \left(\sum_m D^l(g)_{mm'} D^l(g)_{mm''}\right) \frac{\partial E}{\partial h_{j,klm'}^{(t)}} h_{j,klm''}^{(t)}
\end{align}

For the real representations used here, $\mathbf{D}^l(g)$ is orthogonal, so $\sum_m D^l(g)_{mm'} D^l(g)_{mm''} = \delta_{m'm''}$. Therefore:
\begin{equation}
\sum_m \frac{\partial E}{\partial h_{j,klm}^{(t)}} \cdot h_{j,klm}^{(t)} \to \sum_{m'} \frac{\partial E}{\partial h_{j,klm'}^{(t)}} h_{j,klm'}^{(t)}
\end{equation}

This shows that the sum over $m$ is invariant under rotation. Since the absolute value and expectation operations preserve this invariance, the importance score $I_{k,l}^{(t)}$ is SO(3) invariant.
\end{proof}

This invariance property ensures that our pruning method produces consistent results under rotations of the molecular system in the calibration dataset.

\subsection{Spectral interpretation of feature order L} \label{AppSec:Spectral Interpretation}

Understanding the physical meaning of the order $l$ helps justify why pruning higher-order features corresponds to reducing the angular resolution of the model.

\begin{proposition} \label{prop:frequency_l}
The feature order $l$ corresponds to the angular frequency on the sphere. Higher $l$ implies higher spatial frequency, capturing finer angular details of the atomic environment.
\end{proposition}

\begin{proof}
The angular dependence of the atomic features is encoded using Spherical Harmonics $Y_{lm}(\theta, \phi)$. These functions are the eigenfunctions of the Laplace operator $\Delta_{S^2}$ on the unit sphere:
\begin{equation}
    -\Delta_{S^2} Y_{lm}(\theta, \phi) = l(l+1) Y_{lm}(\theta, \phi)
\end{equation}
In harmonic analysis on the sphere, the eigenvalue $\lambda_l = l(l+1)$ is analogous to the squared frequency $|\omega|^2$ in Fourier analysis on Euclidean space.
A function $f(\theta, \phi)$ composed of spherical harmonics up to degree $L_{\text{max}}$ has a band-limit determined by $L_{\text{max}}$.
Increasing $l$ increases the number of nodal lines on the sphere, allowing the function to represent more rapid changes in angular direction. Therefore, pruning features with high $l$ is spectrally equivalent to applying a low-pass filter to the angular representation of the local atomic environment, smoothing out high-frequency geometric details while retaining the coarse structural information.
\end{proof}

To empirically validate this spectral interpretation, we present the feature importance analysis using our method in Figure \ref{Fig:Feature Importance Analysis across Orders L}. The experimental results reveal a consistent decay in feature importance as the order $l$ increases. Specifically, the $l=0$ features (low frequency) exhibit the highest importance scores, whereas higher-order features ($l=1, 2, 3$) show progressively lower significance. This observation aligns with Proposition \ref{prop:frequency_l}, confirming that the model naturally prioritizes coarse-grained geometric information. Consequently, pruning high-$l$ features is physically justified.

\begin{figure}[ht]
    \centering
    \includegraphics[width=\textwidth]{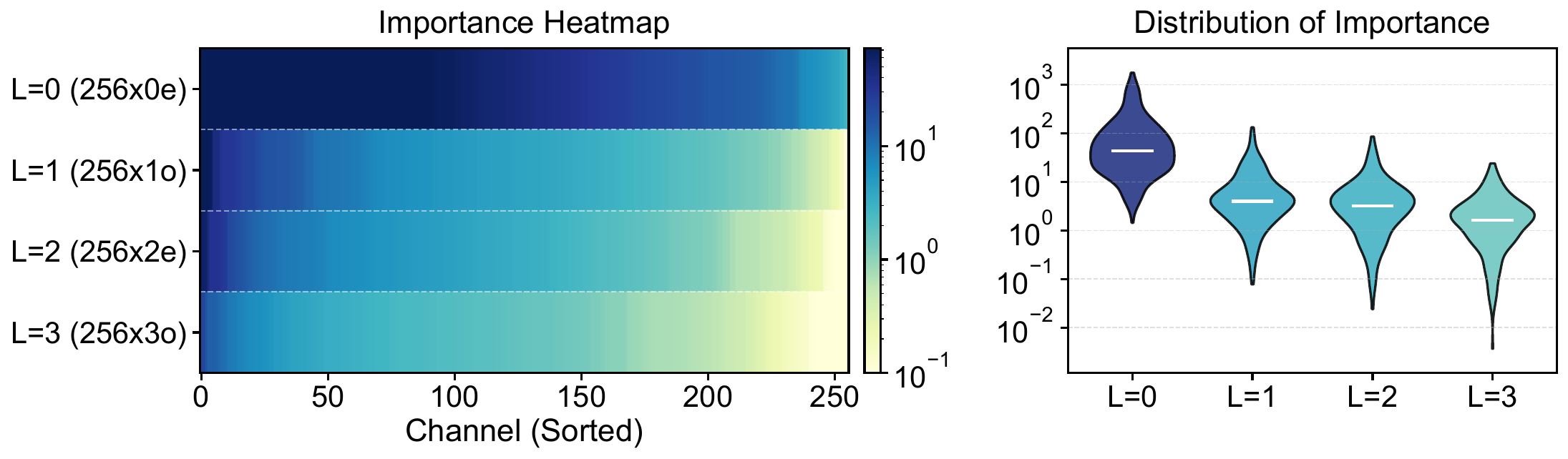}
    \caption{\textbf{Feature importance analysis across orders $L$.} The left panel displays the heatmap of feature importance sorted by channel for each order $L$, while the right panel shows the corresponding distribution (violin plots).}
    \label{Fig:Feature Importance Analysis across Orders L}
\end{figure}

\section{Comparison of depth and width pruning} \label{AppSec:Comparison of Depth and Width Pruning}

We perform additional comparisons of Depth (reducing the number of message passing layers) and Width pruning (reducing the order $L$ of the internal features) on the MACE architecture.

\subsection{Ablation study: non-linear readout substitution in depth pruning}

In the standard MACE architecture, the energy is decomposed into contributions from each layer:
\begin{equation}
    \begin{split}
        E_{i} &= E^{(0)}_i + E^{(1)}_i + ... + E^{(T)}_i, \qquad \text{where}\\
        E_i^{(t)} &= \mathcal{R}_t \left( \boldsymbol{h}_i^{(t)} \right) =
        \begin{cases}
            \sum_{\tilde{k}}W^{(t)}_{\text{readout}, \tilde{k}} h^{(t)}_{i, \tilde{k}00}    & \text{if} \;\; t < T \\[13pt]
            {\rm{MLP}}^{(t)}_\text{readout} \left(\left\{h^{(t)}_{i, k00}\right\}_k\right)  & \text{if} \;\; t = T
        \end{cases}
    \end{split}
\end{equation}
where the readout function $\mathcal{R}_t$ is linear for all intermediate layers ($t < T$) and an MLP (non-linear) only for the final layer ($t = T$). This design ensures that the final layer can approximate higher-order residual terms effectively.

When performing Depth Pruning (e.g., pruning a 2-layer model down to 1 layer), simply removing the second layer leaves the first layer with its original linear readout. We hypothesize that this significantly limits the expressivity of the pruned model. To address this, we propose a Readout Substitution strategy: when the last layer is pruned, we replace the linear readout of the new last layer with a non-linear MLP, matching the architecture of a standard $T-1$ layer MACE model.

Table \ref{Tab:Ablation of Readout Substitution in Depth pruning on MACE-MP} presents the ablation study of this substitution. We observe that retaining the linear readout results in high errors (MAE E: 64.0 meV/atom). However, applying the non-linear readout substitution yields a substantial improvement, reducing the Energy MAE by 26.3\% (to 47.2 meV/atom) and Force MAE by 10.7\%. This confirms that the non-linearity in the final readout is critical for the MACE architecture to resolve atomic energies accurately.

\begin{table}[ht]
    \caption{\textbf{Ablation of readout substitution in depth pruning on MACE-MP.} Both variants prune Layer2 L2 C128 to Layer1 L2 C128 and use 9\% of the training data. We evaluate Energy (E, meV/atom), Force (F, meV/\AA), and Stress (S, meV/$\text{\AA}^3$) errors on MPtrj.}
    \label{Tab:Ablation of Readout Substitution in Depth pruning on MACE-MP}
    \centering
    \small
    \setlength{\tabcolsep}{7pt}\renewcommand{\arraystretch}{1.15}
    
    \begin{tabular}{lcccc}
        \toprule
        \textbf{Readout} & \textbf{Params} & \textbf{MAE E} & \textbf{MAE F} & \textbf{MAE S} \\
        \midrule
        Linear & 3,453,696 & 64.0 & 86.6 & 2.7 \\
        Non-linear & 3,455,632 & \textbf{47.2} & \textbf{77.3} & \textbf{2.4} \\
        \bottomrule
    \end{tabular}
    
\end{table}

\subsection{Comparison results}

We further compare the performance of Depth Pruning (with the optimal non-linear readout) against Width Pruning. For Width Pruning, we reduce the irreducible representations of the features from $L=2$ to $L=0$, effectively reducing the geometric tensor interactions while maintaining the message passing depth. We compare these pruned models against From Scratch baselines trained directly with the target architecture sizes.

As shown in Table \ref{Tab:Comparison of Depth and Width pruning methods on MACE-MP}, Width Pruning demonstrates superior performance. The width-pruned model (Layer 2 L0) achieves an Energy MAE of 30.9 meV/atom, which is competitive with the baseline trained from scratch (29.6 meV/atom). This suggests that initializing with higher-order equivariant features ($L=2$) allows the model to learn a robust underlying representation that can be effectively transferred into invariant features ($L=0$).

Conversely, Depth Pruning fails to match the baseline performance. The depth-pruned model (Layer 1 L2) exhibits an Energy MAE of 44.4 meV/atom, a 50\% increase in error compared to the comparable scratch baseline. This disparity highlights a fundamental characteristic of the message-passing architecture: the receptive field is tightly coupled to the number of layers, as established in Proposition \ref{prop:receptive_field}. Depth Pruning truncates the message passing iterations, linearly reducing the physical receptive field ($R_{\text{rec}} = T \times r_c$). This loss of long-range spatial information proves critical for accuracy, whereas Width Pruning retains the full receptive field while only simplifying the local feature representation.

\begin{table}[ht]
    \caption{\textbf{Comparison of depth and width pruning methods on MACE-MP.} The table compares the performance of pruned models against baselines trained from scratch. Metrics include MAE for Energy (E, meV/atom), Force (F, meV/\AA), and Stress (S, meV/$\text{\AA}^3$) on MPtrj, and Energy (E, meV/atom) on the sAlex subset. Data indicates the fraction of training data used. Values in parentheses show the relative change in MAE compared to from-scratch trained models of similar size. Bold values indicate that the pruned model outperforms the from-scratch trained baseline.}
    \label{Tab:Comparison of Depth and Width pruning methods on MACE-MP}
    \centering
    \small
    \setlength{\tabcolsep}{2pt}\renewcommand{\arraystretch}{1.15}
    
    \begin{tabular}{lcccccccc}
        \toprule
        \multirow{2}{*}{\textbf{Method}} & \multirow{2}{*}{\makecell{\textbf{Base}\\\textbf{Model}}} & \multirow{2}{*}{\makecell{\textbf{Pruned}\\\textbf{Model}}} & \multirow{2}{*}{\textbf{Params}} & \multirow{2}{*}{\textbf{Data}} & \multicolumn{3}{c}{\textbf{MPtrj}} & \textbf{sAlex} \\
        \cmidrule(lr){6-8} \cmidrule(lr){9-9}
        & & & & & \textbf{MAE E} & \textbf{MAE F} & \textbf{MAE S} & \textbf{MAE E} \\
        \midrule
        \multirow{4}{*}{\makecell[l]{From\\Scratch}}         & Layer2 L1 C256      & \multicolumn{1}{c}{\multirow{4}{*}{/}}    & 15,847,440       & \multirow{4}{*}{100\%}  & /                       & /                       & /                      & 81.7                   \\
                                              & Layer2 L2 C128      &                       & 5,725,072        &                      & 22.6                    & 42.2                    & 1.6                    & 65.1                   \\
                                              & Layer2 L1 C128      &                       & 4,688,656        &                      & 25.8                    & 46.7                    & 1.7                    & 71.4                   \\
                                              & Layer2 L0 C128      &                       & 3,847,696        &                      & 29.6                    & 52.3                    & 1.8                    & 78.0                   \\
        \midrule
        \makecell[l]{Depth\\Pruning} & Layer2 L2 C128      & Layer1 L2 C128        & 3,455,632        & \textbf{25\%}          & \makecell{44.4\\(+50.0\%)}          & \makecell{63.1\\(+20.7\%)}          & \makecell{2.1\\(+16.7\%)}          & \makecell{99.0\\(+26.9\%)}         \\
        \makecell[l]{Width\\Pruning} & Layer2 L2 C128      & Layer2 L0 C128        & 3,847,696        & \textbf{25\%}          & \makecell{30.9\\(+4.4\%)} & \makecell{\textbf{50.8}\\\textbf{(-2.9\%)}}  & \makecell{1.9\\(+5.5\%)}           & \makecell{\textbf{74.1}\\\textbf{(-5.0\%)}} \\
        \bottomrule
    \end{tabular}
    
\end{table}

\section{Experimental details} \label{AppSec:Experimental Details}

\subsection{Datasets} \label{AppSec:Datasets}

\paragraph{Inorganic downstream tasks.} We evaluate on eight diverse inorganic datasets: solid-state electrolytes (SSE)~\cite{huang2021deep}, H2O~\cite{zhang2021phase}, AgAu~\cite{wang2021generalizable}, AlMgCu~\cite{jiang2021accurate}, Cu~\cite{zhang2020dp}, Ti~\cite{wen2021specialising}, V~\cite{wang2022classical}, and W~\cite{wang2022tungsten}. These datasets cover simulations of diverse inorganic compounds under complex temperature and pressure conditions.

\paragraph{Organic downstream tasks.} For organic systems, we evaluate on 3BPA~\cite{kovacs2021linear}, a flexible drug-like molecule tested across temperature regimes (300 K, 600 K, 1200 K) and dihedral angle configurations; AcAc~\cite{batatia2022design} (acetylacetone), which probes complex dihedral potential energy surfaces; and rMD17~\cite{christensen2020role}, a benchmark of small organic molecules with high-quality DFT-computed energies and forces.

\subsection{Calibration} \label{AppSec:Calibration}

To compute the importance scores during the calibration phase, we construct representative subsets of the pre-training data tailored to the specific distribution of each dataset. For the MPtrj dataset, which contains approximately 1.5 million configurations derived from roughly 150,000 unique structures, we employ a structure-based sampling approach to ensure chemical diversity and prevent bias toward trajectories with higher frame counts. Specifically, we randomly select $n$ configurations for each unique structure. Setting $n=1$, $n=3$, and $n=5$ results in calibration subsets covering approximately 9\%, 25\%, and 40\% of the total dataset, respectively. For the SPICE dataset, we perform uniform sampling within each data category to generate the calibration subset.

\FloatBarrier
\subsection{Pruning details for MACE architecture}

\paragraph{Order pruning and the $L=0$ target.} In our MACE notation, $L$ specifies the maximum order of the hidden features passed between interaction blocks. At $L=0$, these features are invariant, but higher-order equivariant tensors are still formed within each block and contracted to scalar outputs. To prune Layer2 L2 C128 to Layer2 L0 C128, we remove all inter-block $l>0$ feature blocks and their associated weights, transfer the remaining $l=0$ weights and compatible tensor-product paths, and retrain the compact model. All 128 scalar channels are retained, so the target architecture fixes the order-pruning mask. Importance ranking is used when selecting a subset of channels, not when removing all non-scalar inter-block features at unchanged channel width.

\paragraph{Embedding.} We investigate the impact of edge embedding dimensions when pruning feature channels. The source model in this ablation has 8 Bessel basis functions, while the standard target model has 10. The "Small Embedding" variant retains the source basis size. The "Large Embedding" variant uses the target basis size. The Bessel basis has no learnable weights, but its output is the input to the radial MLP in each interaction block. Changing the basis size from 8 to 10 makes the first MLP layer's input dimension incompatible with the source weights. We therefore randomly re-initialize only the first layer of each radial MLP. The remaining compatible weights are transferred through copying or slicing. As shown in Table~\ref{Tab:Ablation of embedding size in channel pruning on MACE-MP}, the Large Embedding variant reduces Energy MAE from 45.3 to 42.7 meV/atom compared with direct inheritance. We adopt this target-compatible initialization in our main experiments.

\begin{table}[ht]
    \caption{\textbf{Ablation of embedding size in channel pruning on MACE-MP.} Both variants prune Layer2 L1 C256 to Layer2 L1 C128 and use 9\% of the training data. We evaluate Energy (E, meV/atom), Force (F, meV/\AA), and Stress (S, meV/$\text{\AA}^3$) errors on the MPtrj. The channels that are pruned here are randomly selected.}
    \label{Tab:Ablation of embedding size in channel pruning on MACE-MP}
    \centering
    \small
    \setlength{\tabcolsep}{7pt}\renewcommand{\arraystretch}{1.15}
    
    \begin{tabular}{lcccc}
        \toprule
        \textbf{Embedding} & \textbf{Params} & \textbf{MAE E} & \textbf{MAE F} & \textbf{MAE S} \\
        \midrule
        Small & 4,688,400 & 45.3 & 68.7 & \textbf{2.3} \\
        Large & 4,688,656 & \textbf{42.7} & \textbf{68.6} & 2.4 \\
        \bottomrule
    \end{tabular}
    
\end{table}

\paragraph{High body order features.} In the MACE architecture, high body-order features are synthesized via a symmetric contraction of atomic basis functions, mathematically realized through Generalized Clebsch-Gordan (GCG) coefficients. Our pruning method extends naturally to this mechanism by associating the importance score $I_{k, l}^{(t)}$ with the output of these contraction steps. When a feature block $(k, l)$ representing a specific many-body interaction is pruned, we physically slice the corresponding dimensions in the CG coupling tensor during the Structural Alignment phase. This effectively removes the computational graph branches responsible for calculating specific higher-order correlations. By pruning these features based on the energy-force sensitivity criterion, the model selectively discards computationally expensive many-body terms.

\paragraph{Residual connections.} The presence of residual connections requires strict dimensional consistency between the input and output of a computational block. In MACE, this pathway is introduced by an element-dependent linear projection, formally defined as:
\begin{equation}
    \sum_{\tilde{k}} W_{z_i, k \tilde{k} l}^{(t)} h_{i, \tilde{k} l m}^{(t)},
\end{equation}
where $W_{z_i, k \tilde{k} l}^{(t)}$ represents the learnable weights specific to the atomic species $z_i$, mapping input channels $\tilde{k}$ to output channels $k$ for order $l$. During structural alignment, we slice the input and output channel dimensions according to their respective pruning masks. Furthermore, a unique advantage of this element-wise parameterization is the ability to perform species-specific compression. Since the weights are independent for each element type $z_i$, parameters corresponding to chemical elements not present in the target downstream dataset can be explicitly removed prior to fine-tuning. This step, combined with channel slicing, results in a highly compact model specialized for both the geometric structure and the chemical composition of the target system.

\subsection{Ablation study on importance criterion} \label{AppSec:Ablation Importance}

We validate the effectiveness of our proposed importance estimation method by pruning a pre-trained MACE-MP model from Layer2 L1 C256 to Layer2 L1 C128 and performing lightweight retraining using 9\% of the MPtrj dataset. Table~\ref{Tab:Comparison of different importance estimation methods on MACE-MP} compares our method against three baselines using the same pruning granularity: Random Pruning, Magnitude-based Pruning, and Activation-based Pruning.

\begin{table}[ht]
    \caption{\textbf{Comparison of different importance estimation methods on MACE-MP.} The model is pruned from Layer2 L1 C256 to Layer2 L1 C128 and retrained using 9\% of the MPtrj. Metrics include MAE for Energy (E, meV/atom), Force (F, meV/\AA), and Stress (S, meV/$\text{\AA}^3$) on MPtrj. The best results are in bold.}
    \label{Tab:Comparison of different importance estimation methods on MACE-MP}
    \centering
    \small
    \setlength{\tabcolsep}{6pt}\renewcommand{\arraystretch}{1.15}
    \begin{tabular}{lccc}
        \toprule
        \textbf{Method} & \textbf{MAE E} & \textbf{MAE F} & \textbf{MAE S} \\
        \midrule
        Random              & 42.7           & 68.6           & 2.4          \\
        Magnitude           & 47.7           & 68.9           & 2.5          \\
        Activation          & 36.4           & 64.4           & \textbf{2.3} \\
        Our Method & \textbf{35.9}  & \textbf{64.2}  & \textbf{2.3} \\
        \bottomrule
    \end{tabular}
\end{table}

Our physics-informed criterion yields the lowest errors across all metrics. This indicates that measuring the sensitivity of the potential energy surface to specific geometric tensors provides a more robust pruning signal than magnitude or activation alone.

\FloatBarrier
\subsection{Validity of first-order importance estimates} \label{AppSec:First-order Validity}
\paragraph{First-order estimation at increasing pruning ratios.}
We evaluate the first-order approximation on 500 MPtrj validation configurations using a frozen MACE-MP checkpoint. We mask complete $(k,l)$ blocks at the tensor-product output of the second interaction block, removing the same fraction of channels within each order $l=0,1,2,3$. Within each order, we select the lowest-importance blocks using scores computed once in the original model. The same configurations are used for importance estimation, mask selection, and evaluation, without retraining. For each configuration, the actual $\Delta E$ is the masked-model energy minus the original-model energy. The one-shot estimate is the signed sum of the first-order contributions of the removed blocks, evaluated in the original model. Table~\ref{Tab:Taylor Pruning Ratios} reports mean absolute energy changes per atom and the Pearson correlation between the signed predicted and actual changes across configurations. Estimation accuracy decreases as the pruning ratio increases, but the first-order estimates remain strongly correlated with the actual changes.

\begin{table}[ht]
    \caption{\textbf{First-order estimation at increasing pruning ratios.} Predicted and actual changes are mean absolute $\Delta E$ in meV/atom. Pearson correlation uses signed per-configuration changes.}
    \label{Tab:Taylor Pruning Ratios}
    \centering
    \small
    \setlength{\tabcolsep}{9pt}\renewcommand{\arraystretch}{1.15}
    \begin{tabular}{cccc}
        \toprule
        \textbf{Pruning ratio} & \textbf{First-order estimate} & \textbf{Actual} & \textbf{Pearson} \\
        \midrule
        25\% & 1.947 & 2.010 & 0.9968 \\
        50\% & 12.616 & 14.093 & 0.9719 \\
        75\% & 37.490 & 46.625 & 0.8755 \\
        \bottomrule
    \end{tabular}
\end{table}

\paragraph{One-shot versus iterative pruning.}
We also perform iterative pruning in steps of 6.25 percentage points. At each step, we recompute importance scores in the currently masked model, remove the lowest-scoring remaining blocks within each order, and accumulate the stepwise first-order estimates. Table~\ref{Tab:One Shot Iterative} compares the two procedures. Since they can select different blocks, their actual energy changes are reported separately. At higher pruning ratios, iterative estimates are closer to the actual changes, while the actual perturbations are similar to those of one-shot pruning. Considering the additional cost, we use one-shot estimation in our pruning experiments.

\begin{table}[ht]
    \caption{\textbf{One-shot versus iterative pruning.} Values are mean absolute $\Delta E$ in meV/atom, measured without retraining.}
    \label{Tab:One Shot Iterative}
    \centering
    \small
    \setlength{\tabcolsep}{8pt}\renewcommand{\arraystretch}{1.15}
    \begin{tabular}{ccccc}
        \toprule
        \multirow{2}{*}{\textbf{Pruning ratio}} & \multicolumn{2}{c}{\textbf{One-shot}} & \multicolumn{2}{c}{\textbf{Iterative}} \\
        \cmidrule(lr){2-3}\cmidrule(lr){4-5}
        & \textbf{Estimate} & \textbf{Actual} & \textbf{Estimate} & \textbf{Actual} \\
        \midrule
        25\% & 1.947 & 2.010 & 2.501 & 2.579 \\
        50\% & 12.616 & 14.093 & 14.636 & 14.450 \\
        75\% & 37.490 & 46.625 & 45.703 & 46.633 \\
        \bottomrule
    \end{tabular}
\end{table}

\FloatBarrier
\newpage
\subsection{Effect of calibration-set size} \label{AppSec:Calibration Size}
We evaluate how the size of the calibration subset affects importance rankings using the same pre-trained MACE-MP checkpoint. Table~\ref{Tab:Calibration Size} compares rankings computed from the 9\%, 25\%, and 40\% MPtrj subsets constructed as described in Appendix~\ref{AppSec:Calibration}, using the 40\% subset as the reference. The importance rankings are highly consistent across the tested data fractions, with Spearman correlations of 0.99976 and 0.99997 for the 9\% and 25\% subsets, respectively.

\begin{table}[ht]
    \caption{\textbf{Effect of calibration-set size on importance rankings.} Spearman correlations are computed relative to the rankings obtained from the 40\% MPtrj subset.}
    \label{Tab:Calibration Size}
    \centering
    \small
    \setlength{\tabcolsep}{10pt}\renewcommand{\arraystretch}{1.15}
    \begin{tabular}{ccc}
        \toprule
        \textbf{Data fraction} & \textbf{Configurations} & \textbf{Spearman correlation} \\
        \midrule
        9\% & 145,923 & 0.99976 \\
        25\% & 394,397 & 0.99997 \\
        40\% & 620,741 & 1.00000 \\
        \bottomrule
    \end{tabular}
\end{table}

\FloatBarrier
\subsection{Ablation study on retraining} \label{AppSec:Ablation Retraining}

To validate the necessity of the Lightweight Retraining stage, we conduct an ablation study on the AgAu dataset. We compared three strategies: training the compact architecture from scratch with random initialization, directly fine-tuning the pruned model without retraining, and our full pipeline, which includes lightweight retraining on 25\% of the pre-training data.

As shown in Table \ref{Tab:Ablation study of Retraining}, training From Scratch yields poor convergence (RMSE E: 369.1 meV/atom), demonstrating the value of initializing from a pre-trained foundation model. Furthermore, skipping the retraining phase (Pruning – Finetuning) results in higher errors compared to our proposed method. By incorporating Lightweight Retraining, the model effectively recovers from pruning-induced approximation errors, significantly reducing the Energy RMSE from 16.9 to 6.6 meV/atom and Force RMSE from 18.4 to 15.1 meV/\AA. This confirms that a brief adaptation phase is crucial for maximizing the performance of the compact model.

\begin{table}[ht]
    \caption{\textbf{Ablation study of retraining.} We evaluate RMSE for Energy (E, meV/atom) and Force (F, meV/\AA) on the AgAu dataset.}
    \label{Tab:Ablation study of Retraining}
    \centering
    \small
    \setlength{\tabcolsep}{6pt}\renewcommand{\arraystretch}{1.15}
    \begin{tabular}{lcc}
        \toprule
        \textbf{Method}                   & \textbf{RMSE E} & \textbf{RMSE F} \\
        \midrule
        From Scratch                      & 369.1          & 34.5           \\
        \midrule
        Pruning – Finetuning              & 16.9           & 18.4           \\
        Pruning (25\%) – Finetuning       & \textbf{6.6}   & \textbf{15.1}  \\
        \bottomrule
    \end{tabular}
\end{table}

\FloatBarrier
\subsection{Retraining hyperparameters}\label{AppSec:Retraining Hyperparameters}

We adopt the official training configurations provided by MACE for the retraining process. To ensure reproducibility, the detailed hyperparameters are listed in Table \ref{Tab:Hyperparameters}.

The from-scratch baselines and the corresponding pruned models use the same training hyperparameters and epoch schedules: 200 epochs for MACE-MP and 190 epochs for MACE-OFF, without additional pruning-specific tuning. The MACE-MP baselines show small changes in error near the end of training. Over the final 25 epochs, Energy and Force MAE change by 2.0\% and 1.3\%, respectively, for the from-scratch Layer2 L0 C128 model, and by 1.5\% and 1.1\% for the from-scratch Layer2 L1 C128 model. These small late-stage changes support convergence of the two baselines.

\begin{table}[ht]
    \caption{\textbf{Retraining hyperparameters.}}
    \label{Tab:Hyperparameters}
    \centering
    \small
    \setlength{\tabcolsep}{7pt}\renewcommand{\arraystretch}{1.15}
    
    \begin{tabular}{lcc}
    \toprule
    \textbf{Hyperparameter} & \textbf{MACE-MP} & \textbf{MACE-OFF} \\
    \midrule
    Epochs & 200 & 190 \\
    Loss & universal & ef \\
    E weight & 1 & 40 \\
    F weight & 10 & 1000 \\
    S weight & 100 & / \\
    Learning rate & 0.005 & 0.01 \\
    Weight decay & 1e-8 & 5e-10 \\
    Patience & 5 & 20 \\
    EMA & 0.995 & 0.99 \\
    Gradient clipping & 100 & 1 \\
    \bottomrule
    \end{tabular}
    
\end{table}

\FloatBarrier
\subsection{Feature importance stability after retraining} \label{AppSec:Importance Stability}

A natural concern is whether the importance distribution across $(k, l)$ blocks shifts significantly after retraining, which could render the initial pruning criterion suboptimal in hindsight. To address this, we analyze the evolution of importance scores for the 128 retained channels in a model pruned from Layer2 L1 C256 to Layer2 L1 C128, comparing scores computed immediately after pruning against those computed after retraining.

Figure~\ref{Fig:Feature Importance Analysis across Orders L} (Appendix~\ref{AppSec:Spectral Interpretation}) shows the overall importance distribution across orders $L$ before pruning. To quantify stability after retraining, we compute the Spearman rank correlation between the pre- and post-retraining importance scores across all 128 retained channels, obtaining a value of \textbf{0.81}. This strong positive correlation indicates that the global relative ranking of feature importance is largely preserved. Furthermore, we examine the overlap of the top-16 channels (top 12.5\%) before and after retraining: \textbf{14 out of 16 channels} (87.5\%) remain identical. This demonstrates that the core structural pathways identified by our gradient-activation criterion remain the most heavily utilized by the retrained model, validating the reliability of the one-shot pruning criterion.

\subsection{End-to-end pipeline cost} \label{AppSec:Pipeline Cost}

Table~\ref{Tab:Pipeline Cost Comparison} compares GPU hours and wall-clock time for our structural pruning pipeline and training a compact foundation model from scratch on NVIDIA A800 GPUs. Pre-training and retraining use four GPUs, while calibration, pruning, and fine-tuning use one GPU.

\begin{table}[ht]
    \caption{\textbf{End-to-end training cost comparison on NVIDIA A800 GPUs.} Our pruning pipeline reduces total GPU hours by $\sim$4$\times$ relative to pre-training the same compact architecture from scratch.}
    \label{Tab:Pipeline Cost Comparison}
    \centering
    \small
    \setlength{\tabcolsep}{2.5pt}\renewcommand{\arraystretch}{1.15}
    \begin{tabular}{llcc}
        \toprule
        \textbf{Strategy} & \textbf{Stage} & \textbf{GPU Hours} & \textbf{Wall-Clock (h)} \\
        \midrule
        \multirow{3}{*}{From Scratch} & Pre-training        & 750 & 188 \\
                                      & Fine-tuning         & 35  & 35  \\
                                      & \textbf{Total}      & \textbf{785} & \textbf{223} \\
        \midrule
        \multirow{4}{*}{Pruning (Ours)} & Calibration \& Pruning & 1   & 1  \\
                                        & Retraining             & 180 & 45 \\
                                        & Fine-tuning            & 35  & 35 \\
                                        & \textbf{Total}         & \textbf{216} & \textbf{81} \\
        \bottomrule
    \end{tabular}
\end{table}

The pruning-retraining phase requires only 181 GPU hours (1 + 180), compared to 750 GPU hours for pre-training the same compact architecture from scratch, directly validating our claimed $\sim$4$\times$ reduction in training cost. The fine-tuning cost is identical in both pipelines. In total, our pipeline reduces end-to-end GPU hours from 785 to 216 and wall-clock time from 223 hours to 81 hours.

Table~\ref{Tab:Training Cost by Architecture} reports the training cost of each MACE-MP architecture. All runs use 200 epochs. The from-scratch models use the full MPtrj dataset, while the pruned L1 and L0 models use 40\% and 25\%, respectively. The L2 model is the source foundation model and has no pruned counterpart in this comparison. Its pre-training cost is shown separately from the cost of training the compact target models.

\begin{table}[ht]
    \caption{\textbf{MACE-MP training cost by architecture.} Values are approximate training GPU-hours, with the fraction of MPtrj in parentheses. Calibration and pruning are excluded.}
    \label{Tab:Training Cost by Architecture}
    \centering
    \small
    \setlength{\tabcolsep}{8pt}\renewcommand{\arraystretch}{1.15}
    \begin{tabular}{ccc}
        \toprule
        \textbf{Architecture} & \makecell{\textbf{From scratch}\\\textbf{GPU-hours (data)}} & \makecell{\textbf{Pruned}\\\textbf{GPU-hours (data)}} \\
        \midrule
        Layer2 L2 C128 & $\sim$5,500 (100\%) & N/A \\
        Layer2 L1 C128 & $\sim$1,800 (100\%) & $\sim$650 (40\%) \\
        Layer2 L0 C128 & $\sim$750 (100\%) & $\sim$180 (25\%) \\
        \bottomrule
    \end{tabular}
\end{table}

For each compact target architecture, the from-scratch and pruned models use the same hyperparameters and epoch count, giving the same training schedule per sample. Total training compute is lower for the pruned models because they use less data. Thus, their accuracy gains do not result from receiving more target-model training compute. For the L0 target, calibration and pruning, including importance computation, add approximately one GPU-hour to the 180 GPU-hours of retraining, giving the 181 GPU-hours reported above.

The nominal 2.5$\times$--4$\times$ MACE-MP training-cost reduction follows from the data fractions at a fixed target architecture and epoch count: $100/40=2.5$ and $100/25=4$. The measured L0 cost ratio, including calibration and pruning, is $750/181\approx4.1$. Parameter reductions use the source-to-target parameter ratios in Table~\ref{Tab:Data efficiency of pruning methods on MACE-MP}. The endpoints are $5{,}725{,}072/3{,}847{,}696\approx1.49$ and $15{,}847{,}440/3{,}847{,}696\approx4.12$, reported as approximately 1.5$\times$--4$\times$.

\subsection{Throughput and memory usage} \label{AppSec:Throughput and Memory Usage}

To validate the efficiency gains achieved through our structural pruning approach, we evaluate the inference throughput and memory consumption of models with varying sizes on an NVIDIA A800 GPU. We conduct two sets of experiments: one with a batch size of 1 to measure latency, and another to maximize the batch size to evaluate peak throughput. The experiments are conducted using the MPtrj dataset, sampling one configuration from each of the approximately 150,000 unique structures. For the maximum batch size tests, we utilize cuEquivariance~\cite{geiger2025accelerate} v0.8.1, which accelerates computation and reduces memory footprint in such memory-bound scenarios.

Table \ref{Tab:Throughput and Memory Usage of models with different sizes} presents the performance comparison. The results demonstrate that structural pruning significantly reduces computational costs. Compared to the baseline, the pruned models achieve a memory reduction of up to 5.7$\times$ and a throughput improvement of up to 2.7$\times$.

Both reported gains compare Layer2 L2 C128 with Layer2 L0 C128. The throughput ratio is $38{,}043/14{,}031\approx2.7$ at each model's maximum batch size. The memory ratio is $10.3/1.8\approx5.7$ at batch size 1.

\begin{table}[ht]
    \caption{\textbf{Throughput and memory usage comparison of models with different sizes.} We evaluate performance in two regimes: single-sample inference (Batch Size = 1) to measure latency, and maximum batch size to measure peak throughput.}
    \label{Tab:Throughput and Memory Usage of models with different sizes}
    \centering
    \small
    \setlength{\tabcolsep}{2.5pt}\renewcommand{\arraystretch}{1.15}
    
    \begin{tabular}{lccccc}
        \toprule
        \multirow{2}{*}{\textbf{Model}} & \multicolumn{2}{c}{\textbf{Batch Size = 1}} & \multicolumn{3}{c}{\textbf{Max Batch Size}} \\
        \cmidrule(lr){2-3} \cmidrule(lr){4-6}
        & \makecell{\textbf{Throughput}\\\textbf{(atom/s)}} & \makecell{\textbf{Memory}\\\textbf{(GB)}} & \textbf{Batch Size} & \makecell{\textbf{Throughput}\\\textbf{(atom/s)}} & \makecell{\textbf{Memory}\\\textbf{(GB)}} \\
        \midrule
        Layer2 L2 C128 & 1,204  & 10.3 & 64  & 14,031  & 58.6 \\
        Layer2 L1 C128 & 1,532  & 4.7  & 128 & 23,879  & 62.4 \\
        Layer2 L0 C128 & 2,017  & 1.8  & 256 & 38,043  & 50.1 \\
        \bottomrule
    \end{tabular}
    
\end{table}

\subsection{MD simulation validation} \label{AppSec:MD Simulation}

While energy and force MAE are standard surrogates for MLIP quality, the ultimate test of a potential is its behavior in actual MD simulations, where small force errors can accumulate over long trajectories. To validate that our pruned models produce stable and physically accurate dynamics, we conduct MD simulations on liquid water at 300 K using the pruned MACE-MP model (25\% data), fine-tuned on the H2O dataset, and compare against the full pre-trained model.

We run two types of simulations: (1) a \textbf{10 ps NVE} simulation to assess energy conservation (energy drift), and (2) a \textbf{20 ps NVT} simulation to evaluate structural properties, specifically the radial distribution functions (RDFs). Results are summarized in Table~\ref{Tab:MD simulation validation}.

\begin{table}[ht]
    \caption{\textbf{MD simulation validation of pruned vs.\ full model on liquid water at 300 K.} RDF peak positions are reported for O–O, O–H, and H–H pairs. The pruned model maintains near-identical energy conservation and structural properties.}
    \label{Tab:MD simulation validation}
    \centering
    \small
    \setlength{\tabcolsep}{2.5pt}\renewcommand{\arraystretch}{1.15}
    \begin{tabular}{lcccc}
        \toprule
        \multirow{2}{*}{\textbf{Model}} & \textbf{NVE Energy Drift} & \multicolumn{3}{c}{\textbf{RDF 1st / 2nd Peak (\AA)}} \\
        \cmidrule(lr){3-5}
        & \textbf{(meV/atom/ps) $\downarrow$} & \textbf{O–O} & \textbf{O–H} & \textbf{H–H} \\
        \midrule
        Pre-trained (100\%) & $<0.01$ & 2.72 / 4.42 & 0.98 / 1.75 & 1.56 / 2.29 \\
        Pruned (25\%)       & $<0.01$ & 2.72 / 4.52 & 0.98 / 1.77 & 1.56 / 2.31 \\
        \bottomrule
    \end{tabular}
\end{table}

Both models achieve near-perfect energy conservation ($<0.01$ meV/atom/ps), confirming a smooth potential energy surface. The pruned model faithfully reproduces the first-shell RDF peak positions for all three pair types (Figure~\ref{Fig:RDF Compare}), indicating that the local liquid structure is well preserved. These results confirm that structural pruning does not destabilize MD trajectories and produces dynamics that are physically consistent with the full model.

\begin{figure}[ht]
    \centering
    \includegraphics[width=\columnwidth]{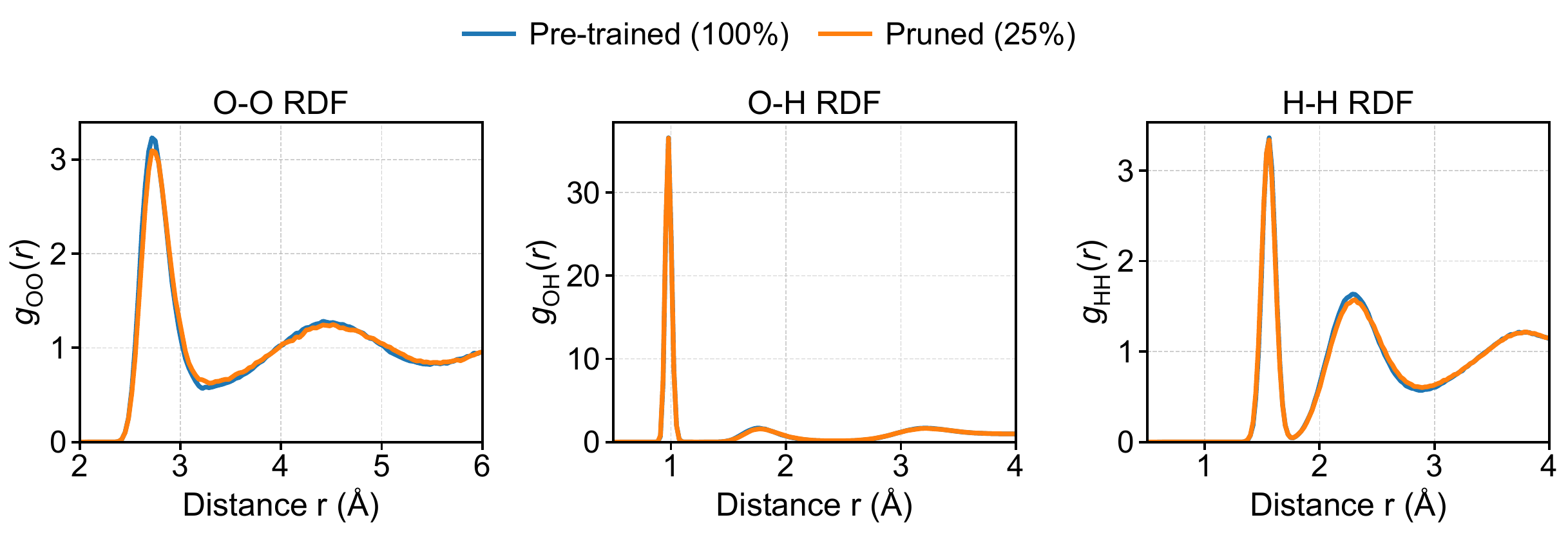}
    \caption{\textbf{Radial distribution functions from NVT MD simulation of liquid water at 300 K.} O–O, O–H, and H–H RDFs are shown for the pre-trained model (100\%, blue) and the pruned model (25\%, orange), both fine-tuned on H2O. Percentages indicate the fraction of MPtrj used for pre-training or recovery training. The two curves overlap closely, confirming that structural pruning preserves the accuracy of local liquid-state structure.}
    \label{Fig:RDF Compare}
\end{figure}

\FloatBarrier
\subsection{MACE-MP downstream fine-tuning} \label{AppSec:MACE-MP Downstream}

Table~\ref{Tab:Effectiveness of pruned MACE-MP model as initialization for downstream fine-tuning} reports the full RMSE values corresponding to Figure~\ref{Fig:Downstream Fine-tuning}. Percentages indicate the fraction of MPtrj used before downstream fine-tuning.

\begin{table}[ht]
    \caption{\textbf{Effectiveness of pruned MACE-MP model as initialization for downstream fine-tuning.} RMSE for Energy (E, meV/atom) and Force (F, meV/\AA) across 8 inorganic datasets. Bold: best per column.}
    \label{Tab:Effectiveness of pruned MACE-MP model as initialization for downstream fine-tuning}
    \centering
    \small
    \setlength{\tabcolsep}{2.5pt}\renewcommand{\arraystretch}{1.15}
    
    \begin{tabular}{llcccccccc}
    \toprule
    \multicolumn{2}{c}{\textbf{Init.}} & \textbf{SSE-PBE} & \textbf{H2O-PD} & \textbf{AgAu} & \textbf{AlMgCu} & \textbf{Cu} & \textbf{Ti} & \textbf{V} & \textbf{W} \\
    \midrule
    \multirow{2}{*}{From Scratch}
    & E & 1.8            & 79.9          & 369.1         & 7.7           & 38.8          & 8.3          & 14.2          & 15.6          \\
    & F & 29.9           & \textbf{29.7} & 34.5          & 42.9          & 13.6          & 94.2         & 140.4         & 181.2         \\
    \midrule
    \multirow{2}{*}{Pre-trained (100\%)}
    & E & 0.6            & 1.1           & 15.4          & 3.4           & \textbf{0.7}  & 6.7          & \textbf{5.4}  & \textbf{6.2}  \\
    & F & 29.1           & 36.8          & 15.8          & 12.2          & 7.4           & 91.4         & 83.5          & 104.1         \\
    \midrule
    \multirow{2}{*}{\textbf{Pruned (25\%)}}
    & E & \textbf{0.5}   & \textbf{0.9}  & \textbf{6.6}  & \textbf{3.2}  & \textbf{0.7}  & \textbf{6.5} & 6.1           & 6.8           \\
    & F & \textbf{25.9}  & 31.6          & \textbf{15.1} & \textbf{11.6} & \textbf{6.9}  & \textbf{90.6}& \textbf{82.3} & \textbf{101.1}\\
    \bottomrule
    \end{tabular}
    
\end{table}

\FloatBarrier
\noindent\begin{minipage}{\linewidth}
For energy and forces separately, we calculate the reported downstream error reduction as an equal-weight average over these eight datasets: $\frac{1}{8}\sum_{d=1}^{8}(1-e_d^{\mathrm{pruned}}/e_d^{\mathrm{scratch}})\times100\%$. Here, $e_d^{\mathrm{pruned}}$ and $e_d^{\mathrm{scratch}}$ are the RMSE values for the Pruned (25\%) and From Scratch rows in Table~\ref{Tab:Effectiveness of pruned MACE-MP model as initialization for downstream fine-tuning}. This gives reductions of 70.1\% for energy and 34.4\% for forces.
\end{minipage}

\FloatBarrier
\subsection{Organic molecule benchmarks} \label{AppSec:MACE-OFF Downstream}

We evaluate the pruned MACE-OFF model on three organic benchmarks covering downstream fine-tuning and ultra-low data regimes.

\paragraph{Downstream fine-tuning on 3BPA and AcAc.} Tables~\ref{Tab:Effectiveness of pruned MACE-OFF model as initialization for 3BPA fine-tuning} and~\ref{Tab:Effectiveness of pruned MACE-OFF model as initialization for acetylacetone fine-tuning} report RMSE for energy and forces on 3BPA~\cite{kovacs2021linear} (a flexible drug-like molecule tested across temperature regimes and dihedral configurations) and AcAc~\cite{batatia2022design}. Using only 50\% of the SPICE pre-training data, the pruned model matches or exceeds the from-scratch baseline in most conditions, confirming that compact specialized models retain the expressivity needed for complex organic potential energy surfaces.

\begin{table}[ht]
    \caption{\textbf{Effectiveness of pruned MACE-OFF model as initialization for 3BPA fine-tuning.} We evaluate RMSE for Energy (E, meV) and Force (F, meV/\AA). The training set is collected at 300 K. Standard deviations are computed over three runs and shown in parentheses. The best results are in bold.}
    \label{Tab:Effectiveness of pruned MACE-OFF model as initialization for 3BPA fine-tuning}
    \centering
    \small
    \setlength{\tabcolsep}{6pt}\renewcommand{\arraystretch}{1.15}
    \begin{tabular}{llccc}
    \toprule
    \multicolumn{2}{c}{\multirow{2}{*}{\textbf{Dataset}}} & \multicolumn{3}{c}{\textbf{Initialization}} \\
    \cmidrule(l){3-5}
    \multicolumn{2}{c}{} & From Scratch & Pre-trained (100\%) & \textbf{Pruned (50\%)} \\
    \midrule
    \multirow{2}{*}{300 K}
    & E  & \textbf{3.0 (0.2)}     & 3.3 (0.03)           & \textbf{3.0 (0.08)}  \\
    & F  & 8.8 (0.3)              &\textbf{7.8 (0.01)}   & \textbf{7.8 (0.07)}  \\
    \midrule
    \multirow{2}{*}{600 K}
    & E  & 9.7 (0.5)              & 7.3 (0.04)           & \textbf{7.1 (0.12)}  \\
    & F  & 21.8 (0.6)             & \textbf{16.6 (0.05)} & 16.7 (0.13)          \\
    \midrule
    \multirow{2}{*}{1200 K}
    & E  & 29.8 (1.0)             & \textbf{20.3 (0.17)} & 21.7 (0.96)          \\
    & F  & 62.0 (0.7)             & \textbf{48.7 (0.56)} & \textbf{48.7 (0.7)}  \\
    \midrule
    \multirow{2}{*}{Dihedral Slices}
    & E  & 7.8 (0.6)              & \textbf{7.3 (0.28)}  & 7.6 (0.25)           \\
    & F  & 16.5 (1.7)             & 12.3 (0.10)          & \textbf{11.8 (0.11)} \\
    \bottomrule
    \end{tabular}
\end{table}

\begin{table}[ht]
    \caption{\textbf{Effectiveness of pruned MACE-OFF model as initialization for AcAc fine-tuning.} We evaluate RMSE for Energy (E, meV) and Force (F, meV/\AA). The training set is collected at 300 K. Standard deviations are computed over three runs and shown in parentheses. The best results are in bold.}
    \label{Tab:Effectiveness of pruned MACE-OFF model as initialization for acetylacetone fine-tuning}
    \centering
    \small
    \setlength{\tabcolsep}{6pt}\renewcommand{\arraystretch}{1.15}
    \begin{tabular}{llccc}
    \toprule
    \multicolumn{2}{c}{\multirow{2}{*}{\textbf{Dataset}}} & \multicolumn{3}{c}{\textbf{Initialization}} \\
    \cmidrule(l){3-5}
    \multicolumn{2}{c}{} & From Scratch & Pre-trained (100\%) & \textbf{Pruned (50\%)} \\
    \midrule
    \multirow{2}{*}{300 K}
    & E                & \textbf{0.9 (0.03)} & 1.0 (0.02)          & \textbf{0.9 (0.02)}   \\
    & F                & \textbf{5.1 (0.10)} & \textbf{5.1 (0.07)} & 5.3 (0.07)            \\
    \midrule
    \multirow{2}{*}{600 K}
    & E                & \textbf{4.6 (0.3)} & 5.8 (0.28)           & 5.6 (0.02)            \\
    & F                & 22.4 (0.9) 	    & \textbf{16.4 (0.70)} & \textbf{16.4 (0.22)}  \\
    \bottomrule
    \end{tabular}
\end{table}

\paragraph{rMD17.} To further evaluate data efficiency in the ultra-low data regime, we benchmark on rMD17~\cite{christensen2020role} following~\cite{batatia2022mace}, training each model on only 50 molecular conformations. Table \ref{Tab:MAE of energy and force on the rMD17 dataset} compares our pruned MACE-OFF model (retraining using 50\% data) against ACE~\cite{kovacs2021linear}, NequIP~\cite{batzner20223}, PACE~\cite{xu2024equivariant}, and a MACE model trained from scratch. The pruned model consistently outperforms from-scratch baselines, indicating that structural pruning retains robust feature representations that generalize to extremely limited fine-tuning data.

\begin{table}[ht]
    \caption{\textbf{MAE of energy (E, meV) and force (F, meV/\AA) on the rMD17 dataset.} Each model is trained using only 50 molecular conformations per molecule. The best results are in bold.}
    \label{Tab:MAE of energy and force on the rMD17 dataset}
    \centering
    \small
    \setlength{\tabcolsep}{6pt}\renewcommand{\arraystretch}{1.15}
    \begin{tabular}{llccccc}
    \toprule
    \multicolumn{2}{c}{\multirow{2}{*}{\textbf{Dataset}}} & \multicolumn{4}{c}{\textbf{From Scratch}} & \textbf{Pruned (50\%)} \\
    \cmidrule(lr){3-6} \cmidrule(lr){7-7}
    & & ACE & NequIP & PACE & MACE & MACE \\
    \midrule
    \multirow{2}{*}{Aspirin}
    & E   & 26.2         & 19.5         & 15.7         & 17.0         & \textbf{10.8}  \\
    & F   & 63.8         & 52.0         & 37.4         & 43.9         & \textbf{26.4}  \\
    \midrule
    \multirow{2}{*}{Azobenzene}
    & E   & 9.0          & 6.0          & 6.7          & \textbf{5.4} & 5.5            \\
    & F   & 28.8         & 20.0         & 17.5         & 17.7         & \textbf{16.6}  \\
    \midrule
    \multirow{2}{*}{Benzene}
    & E   & \textbf{0.2} & 0.6          & 0.6          & 0.7          & 0.6            \\
    & F   & \textbf{2.7} & 2.9          & 3.3          & \textbf{2.7} & 2.9            \\
    \midrule
    \multirow{2}{*}{Ethanol}
    & E   & 8.6          & 8.7          & 6.3          & 6.7          & \textbf{3.0}   \\
    & F   & 43.0         & 40.2         & 25.4         & 32.6         & \textbf{15.7}  \\
    \midrule
    \multirow{2}{*}{Malonaldehyde}
    & E   & 12.8         & 12.7         & 11.5         & 10.0         & \textbf{6.7}   \\
    & F   & 63.5         & 52.5         & 57.3         & 43.3         & \textbf{25.2}  \\
    \midrule
    \multirow{2}{*}{Naphthalene}
    & E   & 3.8          & \textbf{2.1} & \textbf{2.1} & \textbf{2.1} & \textbf{2.1}   \\
    & F   & 19.7         & 10.0         & 9.7          & 9.2          & \textbf{8.9}   \\
    \midrule
    \multirow{2}{*}{Paracetamol}
    & E   & 13.6         & 14.3         & 10.1         & 9.7          & \textbf{6.7}   \\
    & F   & 45.7         & 39.7         & 29.3         & 31.5         & \textbf{22.1}  \\
    \midrule
    \multirow{2}{*}{Salicylic acid}
    & E   & 8.9          & 8.0          & 7.0          & 6.5          & \textbf{4.5}   \\
    & F   & 41.7         & 35.0         & 29.2         & 28.4         & \textbf{19.5}  \\
    \midrule
    \multirow{2}{*}{Toluene}
    & E   & 5.3          & 3.3          & 2.7          & 3.1          & \textbf{2.1}   \\
    & F   & 27.1         & 15.1         & 12.0         & 12.1         & \textbf{9.6}   \\
    \midrule
    \multirow{2}{*}{Uracil}
    & E   & 6.5          & 7.3          & 5.9          & 4.4          & \textbf{3.1}   \\
    & F   & 36.2         & 40.1         & 26.8         & 25.9         & \textbf{16.0}  \\
    \bottomrule
    \end{tabular}
\end{table}

\FloatBarrier
\subsection{Other pre-training datasets}

To verify the transferability of our method to different pre-training data, we conduct an additional experiment using the MACE-OMAT model. The base model was pre-trained on the large-scale OMat24~\cite{barroso2024open} dataset (containing over 100 million structures). We apply our pruning method to this pre-trained model, reducing its equivariant features from $L=1$ to $L=0$, and subsequently retrain it on the official OMat-1M subset.

Table~\ref{Tab:Data efficiency of pruning methods on MACE-OMAT} compares the pruned model against a baseline trained from scratch on OMat-1M with the same architecture. The results show that the pruned model achieves lower errors than the from-scratch baseline, with a 13.0\% reduction in Energy MAE and a 7.1\% reduction in Force MAE. This observation suggests that our pruning approach effectively preserves valuable structural priors. These preserved features serve as a superior initialization compared to random weights, validating the potential of our method in other datasets.

\begin{table}[ht]
    \caption{\textbf{Data efficiency of pruning methods on MACE-OMAT.} Both models use Layer2 L0 C128 with 8,222,244 parameters and are trained on OMat-1M. L pruning uses Layer2 L1 C128 as the source model. We report Energy (E, meV/atom) and Force (F, meV/\AA) MAE on OMat-1M, with relative changes from the from-scratch baseline in parentheses and improvements in bold.}
    \label{Tab:Data efficiency of pruning methods on MACE-OMAT}
    \centering
    \small
    \setlength{\tabcolsep}{7pt}\renewcommand{\arraystretch}{1.15}
    
    \begin{tabular}{lcc}
        \toprule 
        \textbf{Method} & \textbf{MAE E} & \textbf{MAE F} \\
        \midrule
        From Scratch & 26.9 & 106.6 \\
        L Pruning & \textbf{23.4 (-13.0\%)} & \textbf{99.0 (-7.1\%)} \\
        \bottomrule
    \end{tabular}
    
\end{table}

\section{Applicability of the proposed method} \label{AppSec:Applicability of the Proposed Method}

\subsection{SevenNet architecture}

To verify the transferability of our method, we conduct experiments on the SevenNet~\cite{park_scalable_2024} foundation model, which is based on the NequIP~\cite{batzner20223} architecture. We prune the pre-trained SevenNet-l3i5 model to a compact model with the same layer count and channel configuration as the official SevenNet-0 model ($L_{\max}=2$, five interaction layers). Here, SevenNet-0 is the official model name, and the suffix ``0'' does not denote the maximum tensor order. We denote the target architecture as Layer5 L2 C128 in Table~\ref{Tab:Data efficiency of pruning methods on SevenNet}. The following describes our pruning strategy for the non-linear layers involved in the NequIP architecture and the results.

\paragraph{Non-linear activation.} NequIP employs a specific equivariant non-linearity strategy to process geometric tensors. While scalar features ($l=0$) are processed via standard activation functions, higher-order tensors ($l > 0$) rely on a Gated Non-Linearity to preserve equivariance. This mechanism involves generating auxiliary scalar values that act as multiplicative gates for the higher-order features. Consequently, this introduces a strict structural dependency: a higher-order feature block cannot exist without its corresponding scalar gate.

We address this by enforcing a coupling constraint during the mask generation phase. The pruning decision for the gate scalars is not determined by their individual importance scores but is derived deterministically from the masks of the higher-order tensors. Specifically, if a feature block $(k, l)$ with $l > 0$ is retained, we strictly enforce the retention of its associated scalar gate channel in the preceding layer.

\paragraph{Results.} We evaluate our approach by pruning the pre-trained SevenNet-l3i5 model (configured with 5 layers and a maximum order $L_{\text{max}}=3$) down to a target architecture of $L_{\text{max}}=2$. Table~\ref{Tab:Data efficiency of pruning methods on SevenNet} compares our pruning method against a baseline model with the identical architecture ($L_{\text{max}}=2$) trained from scratch.

The results demonstrate significant data efficiency. The model trained from scratch requires 100\% of the dataset to achieve a Force MAE of 40 and an sAlex MAE E of 59.1. In contrast, our Pruning method achieves a better Force MAE of 39 and a lower sAlex MAE E of 58.9 while utilizing only 25\% of the training data. With the full dataset, the pruned model further improves to a Force MAE of 35 and an sAlex MAE E of 53.6, outperforming the from-scratch baseline across all metrics. This indicates that our method effectively transfers structural knowledge from the higher-order foundation model to the compressed model, reducing the data required for convergence compared to standard training.

\begin{table}[ht]
    \caption{\textbf{Data efficiency of pruning methods on SevenNet.} The table compares the performance of pruned models against baselines trained from scratch. Metrics include MAE for Energy (E, meV/atom), Force (F, meV/\AA), and Stress (S, meV/$\text{\AA}^3$) on MPtrj, and Energy (E, meV/atom) on the sAlex subset. Data indicates the fraction of training data used.}
    \label{Tab:Data efficiency of pruning methods on SevenNet}
    \centering
    \small
    \setlength{\tabcolsep}{2.5pt}\renewcommand{\arraystretch}{1.15}
    
    \begin{tabular}{lcccccccc}
        \toprule 
        \multirow{2}{*}{\textbf{Method}} & \multirow{2}{*}{\makecell{\textbf{Base}\\\textbf{Model}}} & \multirow{2}{*}{\makecell{\textbf{Pruned}\\\textbf{Model}}} & \multirow{2}{*}{\textbf{Params}} & \multirow{2}{*}{\textbf{Data}} & \multicolumn{3}{c}{\textbf{MPtrj}} & \textbf{sAlex} \\
        \cmidrule(lr){6-8} \cmidrule(lr){9-9}
        & & & & & \textbf{MAE E} & \textbf{MAE F} & \textbf{MAE S} & \textbf{MAE E} \\
        \midrule
        \multirow{2}{*}{\makecell[l]{From\\Scratch}}   & Layer5 L3 C128                  & \multicolumn{1}{c}{\multirow{2}{*}{/}}               & 1,171,144                 & \multirow{2}{*}{100\%} & /                       & /                       & /               & 51.2                 \\
                                        & Layer5 L2 C128                  &                                  & 842,440                   &                      & 11                      & 40                      & 1.7             & 59.1                 \\
        \midrule
        \multirow{2}{*}{L Pruning}      & \multirow{2}{*}{Layer5 L3 C128} & \multirow{2}{*}{Layer5 L2 C128}  & \multirow{2}{*}{842,440}  & \textbf{25\%}          & 11                      & \textbf{39}             & 2.1             & \textbf{58.9}        \\
                                        &                                 &                                  &                           & 100\%                  & \textbf{9}              & \textbf{35}             & \textbf{1.7}    & \textbf{53.6}        \\
        \bottomrule
    \end{tabular}
    
\end{table}

\subsection{eSCN architecture}

\begin{table}[ht]
    \caption{\textbf{Data efficiency of pruning methods on eSCN.} All models use the target architecture eSCN-L4-M2-Lay12. Metrics are Force MAE (meV/\AA, lower is better) and EFwT (\%, higher is better) on the OC20 validation set (ID split). Bold values indicate that the pruned model outperforms the from-scratch baseline.}
    \label{Tab:Data efficiency of pruning methods on eSCN}
    \centering
    \small
    \setlength{\tabcolsep}{6pt}\renewcommand{\arraystretch}{1.15}
    \begin{tabular}{lccc}
        \toprule
        \textbf{Method} & \textbf{Data} & \textbf{Force MAE (meV/\AA) $\downarrow$} & \textbf{EFwT (\%) $\uparrow$} \\
        \midrule
        From Scratch    & 200k          & 33.1                                   & 0.20                          \\
        L Pruning       & 200k          & \textbf{30.3}                          & \textbf{0.24}                 \\
        \midrule
        From Scratch    & 2M            & 19.1                                   & 2.55                          \\
        L Pruning       & 1M            & \textbf{18.4}                          & \textbf{2.86}                 \\
        L Pruning       & 2M            & \textbf{17.5}                          & \textbf{3.36}                 \\
        \bottomrule
    \end{tabular}
\end{table}

To further validate the generalizability of our method beyond MACE and NequIP-based architectures, we conduct experiments on the eSCN (Equivariant Spherical Channel Network)~\cite{passaro2023reducing} foundation model. Like MACE and SevenNet, eSCN constructs geometric tensors via Clebsch-Gordan tensor products and operates on coupled $(2l+1)$-dimensional irreducible representations. This shared mathematical basis means our $(k,l)$-block pruning principle applies to eSCN without any substantive modification.

\paragraph{Setup.} We prune a pre-trained eSCN-L6-M2-Lay12 ($L_{\max}=6$) model to the target architecture eSCN-L4-M2-Lay12 ($L_{\max}=4$). We evaluate data efficiency by training both a randomly initialized model and the pruned model on the OC20 200k and 2M subsets for the same number of epochs. All evaluated models share the same target architecture (eSCN-L4-M2-Lay12).

\paragraph{Results.} Table~\ref{Tab:Data efficiency of pruning methods on eSCN} reports Force MAE and the Energy-Force within Threshold (EFwT) metric on the OC20 validation set (ID split). The pruned model consistently outperforms the from-scratch baseline at every data scale: on 200k samples, it reduces Force MAE from 0.0331 to 0.0303~eV/\AA\ (+20\% relative EFwT), and on 2M samples from 0.0191 to 0.0175~eV/\AA. Notably, the pruned model trained on only 1M samples (50\% of 2M) achieves a Force MAE of 0.0184~eV/\AA, which is comparable to the from-scratch model trained on the full 2M dataset, demonstrating a 2$\times$ reduction in training data requirements.

\paragraph{Inference efficiency.}
We compare the source eSCN-L6-M2-Lay12 model with the pruned eSCN-L4-M2-Lay12 checkpoint retrained on OC20-2M. Both models are evaluated in FP32 on an NVIDIA A100 80GB GPU using the same batch of 24 configurations from the OC20 validation ID split. Inference includes graph construction and energy and force prediction. Table~\ref{Tab:eSCN Inference Efficiency} reports latency and peak allocated GPU memory. Pruning provides a 1.55$\times$ inference speedup and reduces peak memory by 36.4\%. These results show that structural pruning can provide additional efficiency gains for eSCN, which already uses efficient SO(2)-based convolutions.

\begin{table}[ht]
    \caption{\textbf{Inference efficiency of eSCN on OC20.} Measurements use an A100 GPU and batch size 24.}
    \label{Tab:eSCN Inference Efficiency}
    \centering
    \small
    \setlength{\tabcolsep}{7pt}\renewcommand{\arraystretch}{1.15}
    \begin{tabular}{cccc}
        \toprule
        \textbf{Model} & \makecell{\textbf{Parameters}\\\textbf{(M)}} & \makecell{\textbf{Latency}\\\textbf{(ms)}} & \makecell{\textbf{Peak memory}\\\textbf{(GiB)}} \\
        \midrule
        eSCN-L6-M2-Lay12 & 51.84 & 991.8 & 4.89 \\
        eSCN-L4-M2-Lay12 (pruned) & 36.11 & 641.7 & 3.11 \\
        \bottomrule
    \end{tabular}
\end{table}

\FloatBarrier
\section{Licenses for datasets and models} \label{AppSec:Licenses}

Table~\ref{Tab:Licenses} provides license information for all datasets and pre-trained models used in this work. All assets are used in accordance with their respective licenses.

\begin{table}[h]
    \caption{\textbf{Licenses for datasets and models used in this work.}}
    \label{Tab:Licenses}
    \centering
    \small
    \setlength{\tabcolsep}{2.5pt}\renewcommand{\arraystretch}{1.15}
    \begin{tabular}{lll}
        \toprule
        \textbf{Asset} & \textbf{License} & \textbf{Reference} \\
        \midrule
        \multicolumn{3}{l}{\textit{Pre-training Datasets}} \\
        MPtrj & MIT & \cite{deng2023chgnet} \\
        SPICE & CC0 & \cite{eastman2023spice} \\
        OC20 & CC BY 4.0 & \cite{chanussot2021open} \\
        OMat24 & CC BY 4.0 & \cite{barroso2024open} \\
        \midrule
        \multicolumn{3}{l}{\textit{Downstream Datasets (Inorganic)}} \\
        SSE, H2O, AgAu, AlMgCu, Cu, Ti, V, W & LGPL-3.0 & \cite{huang2021deep, zhang2021phase, wang2021generalizable, jiang2021accurate, zhang2020dp, wen2021specialising, wang2022classical, wang2022tungsten} \\
        sAlex & CC BY 4.0 & \cite{schmidt2024improving, barroso2024open} \\
        Matbench Discovery & CC BY 4.0 & \cite{riebesell2023matbench} \\
        \midrule
        \multicolumn{3}{l}{\textit{Downstream Datasets (Organic)}} \\
        3BPA, AcAc & CC BY 4.0 & \cite{kovacs2021linear, batatia2022design} \\
        rMD17 & CC0 & \cite{christensen2020role} \\
        \midrule
        \multicolumn{3}{l}{\textit{Pre-trained Models}} \\
        MACE-MP-0 & MIT & \cite{batatia2023foundation} \\
        MACE-OFF & ASL License & \cite{kovacs2025mace} \\
        MACE-OMAT & ASL License & \cite{batatia2023foundation} \\
        SevenNet & GPLv3 & \cite{park_scalable_2024} \\
        eSCN & CC BY 4.0 & \cite{passaro2023reducing} \\
        \bottomrule
    \end{tabular}
\end{table}

\FloatBarrier
\section{Broader impacts} \label{AppSec:Broader Impacts}

This work reduces the computational cost of atomistic foundation models, lowering barriers to computational materials discovery and drug design, which benefits researchers with limited resources and reduces energy consumption. There is no direct path to negative societal impacts: atomistic simulation tools are domain-specific scientific instruments not applicable to surveillance, disinformation, or fairness-sensitive decision making, and our compression method does not change the underlying capabilities or introduce new misuse risks beyond those already inherent to molecular simulation.


\newpage
\input{checklist_final.tex}

\end{document}

%% file: checklist_final.tex
\section*{NeurIPS Paper Checklist}

\begin{enumerate}

\item {\bf Claims}
    \item[] Question: Do the main claims made in the abstract and introduction accurately reflect the paper's contributions and scope?
    \item[] Answer: \answerYes{}
    \item[] Justification: The abstract and introduction clearly state all key contributions: the pruning framework, 1.5$\times$--4$\times$ parameter reduction, 2.5$\times$--4$\times$ training cost reduction, 2.7$\times$ inference speedup, 70.1\%/34.4\% downstream error reductions, and composability with quantization and knowledge distillation. These claims are supported by experimental results in Section~4 (Accuracy, Efficiency, Application, and Combination subsections).
    \item[] Guidelines:
    \begin{itemize}
        \item The answer \answerNA{} means that the abstract and introduction do not include the claims made in the paper.
        \item The abstract and/or introduction should clearly state the claims made, including the contributions made in the paper and important assumptions and limitations. A \answerNo{} or \answerNA{} answer to this question will not be perceived well by the reviewers. 
        \item The claims made should match theoretical and experimental results, and reflect how much the results can be expected to generalize to other settings. 
        \item It is fine to include aspirational goals as motivation as long as it is clear that these goals are not attained by the paper. 
    \end{itemize}

\item {\bf Limitations}
    \item[] Question: Does the paper discuss the limitations of the work performed by the authors?
    \item[] Answer: \answerYes{}
    \item[] Justification: The conclusion discusses a key limitation: the target pruned architecture must be specified in advance, which limits flexibility. Neural architecture search is identified as a future direction to address this.
    \item[] Guidelines:
    \begin{itemize}
        \item The answer \answerNA{} means that the paper has no limitation while the answer \answerNo{} means that the paper has limitations, but those are not discussed in the paper. 
        \item The authors are encouraged to create a separate ``Limitations'' section in their paper.
        \item The paper should point out any strong assumptions and how robust the results are to violations of these assumptions (e.g., independence assumptions, noiseless settings, model well-specification, asymptotic approximations only holding locally). The authors should reflect on how these assumptions might be violated in practice and what the implications would be.
        \item The authors should reflect on the scope of the claims made, e.g., if the approach was only tested on a few datasets or with a few runs. In general, empirical results often depend on implicit assumptions, which should be articulated.
        \item The authors should reflect on the factors that influence the performance of the approach. For example, a facial recognition algorithm may perform poorly when image resolution is low or images are taken in low lighting. Or a speech-to-text system might not be used reliably to provide closed captions for online lectures because it fails to handle technical jargon.
        \item The authors should discuss the computational efficiency of the proposed algorithms and how they scale with dataset size.
        \item If applicable, the authors should discuss possible limitations of their approach to address problems of privacy and fairness.
        \item While the authors might fear that complete honesty about limitations might be used by reviewers as grounds for rejection, a worse outcome might be that reviewers discover limitations that aren't acknowledged in the paper. The authors should use their best judgment and recognize that individual actions in favor of transparency play an important role in developing norms that preserve the integrity of the community. Reviewers will be specifically instructed to not penalize honesty concerning limitations.
    \end{itemize}

\item {\bf Theory assumptions and proofs}
    \item[] Question: For each theoretical result, does the paper provide the full set of assumptions and a complete (and correct) proof?
    \item[] Answer: \answerYes{}
    \item[] Justification: Proposition B.2 (SO(3) equivariance preservation under block-wise pruning), Proposition B.3 (SO(3) invariance of the importance criterion), and Proposition B.4 (spectral interpretation of feature order) are provided with complete proofs in Appendix~B.
    \item[] Guidelines:
    \begin{itemize}
        \item The answer \answerNA{} means that the paper does not include theoretical results. 
        \item All the theorems, formulas, and proofs in the paper should be numbered and cross-referenced.
        \item All assumptions should be clearly stated or referenced in the statement of any theorems.
        \item The proofs can either appear in the main paper or the supplemental material, but if they appear in the supplemental material, the authors are encouraged to provide a short proof sketch to provide intuition. 
        \item Inversely, any informal proof provided in the core of the paper should be complemented by formal proofs provided in appendix or supplemental material.
        \item Theorems and Lemmas that the proof relies upon should be properly referenced. 
    \end{itemize}

    \item {\bf Experimental result reproducibility}
    \item[] Question: Does the paper fully disclose all the information needed to reproduce the main experimental results of the paper to the extent that it affects the main claims and/or conclusions of the paper (regardless of whether the code and data are provided or not)?
    \item[] Answer: \answerYes{}
    \item[] Justification: Appendix~D provides full experimental details including dataset descriptions, calibration procedure, pruning configurations, retraining hyperparameters, and hardware specifications. Code is available at \url{https://github.com/HPC-AI-Team/SO3-Pruning}.
    \item[] Guidelines:
    \begin{itemize}
        \item The answer \answerNA{} means that the paper does not include experiments.
        \item If the paper includes experiments, a \answerNo{} answer to this question will not be perceived well by the reviewers: Making the paper reproducible is important, regardless of whether the code and data are provided or not.
        \item If the contribution is a dataset and\slash or model, the authors should describe the steps taken to make their results reproducible or verifiable. 
        \item Depending on the contribution, reproducibility can be accomplished in various ways. For example, if the contribution is a novel architecture, describing the architecture fully might suffice, or if the contribution is a specific model and empirical evaluation, it may be necessary to either make it possible for others to replicate the model with the same dataset, or provide access to the model. In general. releasing code and data is often one good way to accomplish this, but reproducibility can also be provided via detailed instructions for how to replicate the results, access to a hosted model (e.g., in the case of a large language model), releasing of a model checkpoint, or other means that are appropriate to the research performed.
        \item While NeurIPS does not require releasing code, the conference does require all submissions to provide some reasonable avenue for reproducibility, which may depend on the nature of the contribution. For example
        \begin{enumerate}
            \item If the contribution is primarily a new algorithm, the paper should make it clear how to reproduce that algorithm.
            \item If the contribution is primarily a new model architecture, the paper should describe the architecture clearly and fully.
            \item If the contribution is a new model (e.g., a large language model), then there should either be a way to access this model for reproducing the results or a way to reproduce the model (e.g., with an open-source dataset or instructions for how to construct the dataset).
            \item We recognize that reproducibility may be tricky in some cases, in which case authors are welcome to describe the particular way they provide for reproducibility. In the case of closed-source models, it may be that access to the model is limited in some way (e.g., to registered users), but it should be possible for other researchers to have some path to reproducing or verifying the results.
        \end{enumerate}
    \end{itemize}

\item {\bf Open access to data and code}
    \item[] Question: Does the paper provide open access to the data and code, with sufficient instructions to faithfully reproduce the main experimental results, as described in supplemental material?
    \item[] Answer: \answerYes{}
    \item[] Justification: Code is available at \url{https://github.com/HPC-AI-Team/SO3-Pruning}. All datasets used are publicly available and properly cited.
    \item[] Guidelines:
    \begin{itemize}
        \item The answer \answerNA{} means that paper does not include experiments requiring code.
        \item Please see the NeurIPS code and data submission guidelines (\url{https://neurips.cc/public/guides/CodeSubmissionPolicy}) for more details.
        \item While we encourage the release of code and data, we understand that this might not be possible, so \answerNo{} is an acceptable answer. Papers cannot be rejected simply for not including code, unless this is central to the contribution (e.g., for a new open-source benchmark).
        \item The instructions should contain the exact command and environment needed to run to reproduce the results. See the NeurIPS code and data submission guidelines (\url{https://neurips.cc/public/guides/CodeSubmissionPolicy}) for more details.
        \item The authors should provide instructions on data access and preparation, including how to access the raw data, preprocessed data, intermediate data, and generated data, etc.
        \item The authors should provide scripts to reproduce all experimental results for the new proposed method and baselines. If only a subset of experiments are reproducible, they should state which ones are omitted from the script and why.
        \item At submission time, to preserve anonymity, the authors should release anonymized versions (if applicable).
        \item Providing as much information as possible in supplemental material (appended to the paper) is recommended, but including URLs to data and code is permitted.
    \end{itemize}

\item {\bf Experimental setting/details}
    \item[] Question: Does the paper specify all the training and test details (e.g., data splits, hyperparameters, how they were chosen, type of optimizer) necessary to understand the results?
    \item[] Answer: \answerYes{}
    \item[] Justification: Section~4.1 describes the foundation models, datasets, and downstream tasks. Appendix~D describes calibration, pruning configurations, and MACE retraining settings. Appendix~E provides the SevenNet and eSCN experimental setups.
    \item[] Guidelines:
    \begin{itemize}
        \item The answer \answerNA{} means that the paper does not include experiments.
        \item The experimental setting should be presented in the core of the paper to a level of detail that is necessary to appreciate the results and make sense of them.
        \item The full details can be provided either with the code, in appendix, or as supplemental material.
    \end{itemize}

\item {\bf Experiment statistical significance}
    \item[] Question: Does the paper report error bars suitably and correctly defined or other appropriate information about the statistical significance of the experiments?
    \item[] Answer: \answerYes{}
    \item[] Justification: Standard deviations over three independent runs are reported for the downstream fine-tuning experiments on 3BPA and AcAc (Tables in Appendix~D). For the large-scale pre-training experiments (MACE-MP, MACE-OFF, SevenNet, eSCN), error bars are not reported as multiple runs are computationally prohibitive (hundreds of GPU-hours per run). Results are instead validated across multiple datasets and architectures.
    \item[] Guidelines:
    \begin{itemize}
        \item The answer \answerNA{} means that the paper does not include experiments.
        \item The authors should answer \answerYes{} if the results are accompanied by error bars, confidence intervals, or statistical significance tests, at least for the experiments that support the main claims of the paper.
        \item The factors of variability that the error bars are capturing should be clearly stated (for example, train/test split, initialization, random drawing of some parameter, or overall run with given experimental conditions).
        \item The method for calculating the error bars should be explained (closed form formula, call to a library function, bootstrap, etc.)
        \item The assumptions made should be given (e.g., Normally distributed errors).
        \item It should be clear whether the error bar is the standard deviation or the standard error of the mean.
        \item It is OK to report 1-sigma error bars, but one should state it. The authors should preferably report a 2-sigma error bar than state that they have a 96\% CI, if the hypothesis of Normality of errors is not verified.
        \item For asymmetric distributions, the authors should be careful not to show in tables or figures symmetric error bars that would yield results that are out of range (e.g., negative error rates).
        \item If error bars are reported in tables or plots, the authors should explain in the text how they were calculated and reference the corresponding figures or tables in the text.
    \end{itemize}

\item {\bf Experiments compute resources}
    \item[] Question: For each experiment, does the paper provide sufficient information on the computer resources (type of compute workers, memory, time of execution) needed to reproduce the experiments?
    \item[] Answer: \answerYes{}
    \item[] Justification: Appendix~D.10 reports MACE-MP training costs and GPU configurations. Appendices~D.11 and E.2 describe the hardware and measurement settings for the MACE-MP and eSCN inference benchmarks, respectively.
    \item[] Guidelines:
    \begin{itemize}
        \item The answer \answerNA{} means that the paper does not include experiments.
        \item The paper should indicate the type of compute workers CPU or GPU, internal cluster, or cloud provider, including relevant memory and storage.
        \item The paper should provide the amount of compute required for each of the individual experimental runs as well as estimate the total compute. 
        \item The paper should disclose whether the full research project required more compute than the experiments reported in the paper (e.g., preliminary or failed experiments that didn't make it into the paper). 
    \end{itemize}
    
\item {\bf Code of ethics}
    \item[] Question: Does the research conducted in the paper conform, in every respect, with the NeurIPS Code of Ethics \url{https://neurips.cc/public/EthicsGuidelines}?
    \item[] Answer: \answerYes{}
    \item[] Justification: This work involves purely computational methods for compressing atomistic foundation models. It uses only publicly available datasets and open-source model checkpoints. The research conforms to the NeurIPS Code of Ethics in all respects.
    \item[] Guidelines:
    \begin{itemize}
        \item The answer \answerNA{} means that the authors have not reviewed the NeurIPS Code of Ethics.
        \item If the authors answer \answerNo, they should explain the special circumstances that require a deviation from the Code of Ethics.
        \item The authors should make sure to preserve anonymity (e.g., if there is a special consideration due to laws or regulations in their jurisdiction).
    \end{itemize}

\item {\bf Broader impacts}
    \item[] Question: Does the paper discuss both potential positive societal impacts and negative societal impacts of the work performed?
    \item[] Answer: \answerYes{}
    \item[] Justification: Broader impacts are discussed in Appendix~\ref{AppSec:Broader Impacts}. The work has positive societal impact by reducing computational costs and enabling more accessible materials discovery. There is no direct path to negative societal impacts, as atomistic simulation tools are domain-specific scientific instruments not applicable to misuse scenarios such as surveillance or disinformation.
    \item[] Guidelines:
    \begin{itemize}
        \item The answer \answerNA{} means that there is no societal impact of the work performed.
        \item If the authors answer \answerNA{} or \answerNo, they should explain why their work has no societal impact or why the paper does not address societal impact.
        \item Examples of negative societal impacts include potential malicious or unintended uses (e.g., disinformation, generating fake profiles, surveillance), fairness considerations (e.g., deployment of technologies that could make decisions that unfairly impact specific groups), privacy considerations, and security considerations.
        \item The conference expects that many papers will be foundational research and not tied to particular applications, let alone deployments. However, if there is a direct path to any negative applications, the authors should point it out. For example, it is legitimate to point out that an improvement in the quality of generative models could be used to generate Deepfakes for disinformation. On the other hand, it is not needed to point out that a generic algorithm for optimizing neural networks could enable people to train models that generate Deepfakes faster.
        \item The authors should consider possible harms that could arise when the technology is being used as intended and functioning correctly, harms that could arise when the technology is being used as intended but gives incorrect results, and harms following from (intentional or unintentional) misuse of the technology.
        \item If there are negative societal impacts, the authors could also discuss possible mitigation strategies (e.g., gated release of models, providing defenses in addition to attacks, mechanisms for monitoring misuse, mechanisms to monitor how a system learns from feedback over time, improving the efficiency and accessibility of ML).
    \end{itemize}
    
\item {\bf Safeguards}
    \item[] Question: Does the paper describe safeguards that have been put in place for responsible release of data or models that have a high risk for misuse (e.g., pre-trained language models, image generators, or scraped datasets)?
    \item[] Answer: \answerNA{}
    \item[] Justification: Atomistic foundation models for materials simulation pose no high risk for misuse. No special safeguards are required beyond standard open-source release practices.
    \item[] Guidelines:
    \begin{itemize}
        \item The answer \answerNA{} means that the paper poses no such risks.
        \item Released models that have a high risk for misuse or dual-use should be released with necessary safeguards to allow for controlled use of the model, for example by requiring that users adhere to usage guidelines or restrictions to access the model or implementing safety filters. 
        \item Datasets that have been scraped from the Internet could pose safety risks. The authors should describe how they avoided releasing unsafe images.
        \item We recognize that providing effective safeguards is challenging, and many papers do not require this, but we encourage authors to take this into account and make a best faith effort.
    \end{itemize}

\item {\bf Licenses for existing assets}
    \item[] Question: Are the creators or original owners of assets (e.g., code, data, models), used in the paper, properly credited and are the license and terms of use explicitly mentioned and properly respected?
    \item[] Answer: \answerYes{}
    \item[] Justification: All datasets and pre-trained models are properly cited with their original papers. Specific license information for each dataset and model is provided in Appendix~\ref{AppSec:Licenses}.
    \item[] Guidelines:
    \begin{itemize}
        \item The answer \answerNA{} means that the paper does not use existing assets.
        \item The authors should cite the original paper that produced the code package or dataset.
        \item The authors should state which version of the asset is used and, if possible, include a URL.
        \item The name of the license (e.g., CC-BY 4.0) should be included for each asset.
        \item For scraped data from a particular source (e.g., website), the copyright and terms of service of that source should be provided.
        \item If assets are released, the license, copyright information, and terms of use in the package should be provided. For popular datasets, \url{paperswithcode.com/datasets} has curated licenses for some datasets. Their licensing guide can help determine the license of a dataset.
        \item For existing datasets that are re-packaged, both the original license and the license of the derived asset (if it has changed) should be provided.
        \item If this information is not available online, the authors are encouraged to reach out to the asset's creators.
    \end{itemize}

\item {\bf New assets}
    \item[] Question: Are new assets introduced in the paper well documented and is the documentation provided alongside the assets?
    \item[] Answer: \answerYes{}
    \item[] Justification: The pruning code is available at \url{https://github.com/HPC-AI-Team/SO3-Pruning} with documentation covering installation, usage, and reproduction of the main experiments.
    \item[] Guidelines:
    \begin{itemize}
        \item The answer \answerNA{} means that the paper does not release new assets.
        \item Researchers should communicate the details of the dataset\slash code\slash model as part of their submissions via structured templates. This includes details about training, license, limitations, etc. 
        \item The paper should discuss whether and how consent was obtained from people whose asset is used.
        \item At submission time, remember to anonymize your assets (if applicable). You can either create an anonymized URL or include an anonymized zip file.
    \end{itemize}

\item {\bf Crowdsourcing and research with human subjects}
    \item[] Question: For crowdsourcing experiments and research with human subjects, does the paper include the full text of instructions given to participants and screenshots, if applicable, as well as details about compensation (if any)? 
    \item[] Answer: \answerNA{}
    \item[] Justification: The paper does not involve crowdsourcing or research with human subjects.
    \item[] Guidelines:
    \begin{itemize}
        \item The answer \answerNA{} means that the paper does not involve crowdsourcing nor research with human subjects.
        \item Including this information in the supplemental material is fine, but if the main contribution of the paper involves human subjects, then as much detail as possible should be included in the main paper. 
        \item According to the NeurIPS Code of Ethics, workers involved in data collection, curation, or other labor should be paid at least the minimum wage in the country of the data collector. 
    \end{itemize}

\item {\bf Institutional review board (IRB) approvals or equivalent for research with human subjects}
    \item[] Question: Does the paper describe potential risks incurred by study participants, whether such risks were disclosed to the subjects, and whether Institutional Review Board (IRB) approvals (or an equivalent approval/review based on the requirements of your country or institution) were obtained?
    \item[] Answer: \answerNA{}
    \item[] Justification: The paper does not involve human subjects research and therefore requires no IRB approval.
    \item[] Guidelines:
    \begin{itemize}
        \item The answer \answerNA{} means that the paper does not involve crowdsourcing nor research with human subjects.
        \item Depending on the country in which research is conducted, IRB approval (or equivalent) may be required for any human subjects research. If you obtained IRB approval, you should clearly state this in the paper. 
        \item We recognize that the procedures for this may vary significantly between institutions and locations, and we expect authors to adhere to the NeurIPS Code of Ethics and the guidelines for their institution. 
        \item For initial submissions, do not include any information that would break anonymity (if applicable), such as the institution conducting the review.
    \end{itemize}

\item {\bf Declaration of LLM usage}
    \item[] Question: Does the paper describe the usage of LLMs if it is an important, original, or non-standard component of the core methods in this research? Note that if the LLM is used only for writing, editing, or formatting purposes and does \emph{not} impact the core methodology, scientific rigor, or originality of the research, declaration is not required.
    \item[] Answer: \answerNA{}
    \item[] Justification: LLMs are not used as any component of the core methodology. No declaration is required.
    \item[] Guidelines:
    \begin{itemize}
        \item The answer \answerNA{} means that the core method development in this research does not involve LLMs as any important, original, or non-standard components.
        \item Please refer to our LLM policy in the NeurIPS handbook for what should or should not be described.
    \end{itemize}

\end{enumerate}